\documentclass[12pt]{article}
\usepackage{amsmath}
\usepackage{graphicx,psfrag,epsf}
\usepackage{enumerate}
\usepackage{natbib}
\usepackage{url} 

\newcommand{\blind}{1}

\usepackage{comment}
\usepackage{rotating}
\usepackage{amsthm}
\usepackage{amssymb}
\usepackage{setspace}
\usepackage{color}
\usepackage[linesnumbered, ruled]{algorithm2e}
\SetKwRepeat{Do}{do}{while}%

\newcommand\independent{\protect\mathpalette{\protect\independenT}{\perp}}
\def\independenT#1#2{\mathrel{\rlap{$#1#2$}\mkern2mu{#1#2}}}

\newcommand{\T}{\intercal}

\newcommand{\OMIT}[1]{\relax}   
\def\text{{\rm}}

 \newcommand{\bma}[1]{\mbox{\boldmath $#1$}}

 \newcommand{\bS}{ {\bma{S}} }
 \newcommand{\bs}{ {\bma{s}} }

 \newcommand{\bW}{ {\bma{W}} }

\newtheorem{thm}{Theorem}[section]
\newtheorem{lem}[thm]{Lemma}

\theoremstyle{definition}

\newtheorem{rmrk}[thm]{Remark}

\newtheorem{assumption}{Assumption}

\begin{document}

\def\spacingset#1{\renewcommand{\baselinestretch}%
{#1}\small\normalsize} \spacingset{1}

\begin{center}
  \textbf{Estimating Dynamic Treatment Regimes in Mobile Health Using V-learning} \\
	\vspace{0.1in}
  \textbf{ Daniel J.\ Luckett$^{1}$, Eric B.\ Laber$^{2}$, Anna R.\ Kahkoska$^{3}$, \\ 
	David M.\ Maahs$^{4}$, Elizabeth Mayer-Davis$^{3}$, Michael R.\ Kosorok$^{1}$}
  \\
	\vspace{0.1in}
$^1$ Department of Biostatistics, University of North Carolina, Chapel
  Hill, NC 27599 \\ 
$^2$ Department of Statistics, North Carolina State University,
  Raleigh, NC 27695 \\
$^3$ Department of Nutrition, University of North Carolina, Chapel
  Hill, NC 27599 \\ 	
$^4$ Department of Pediatrics, Stanford University, Stanford, CA 94305	
\end{center}

\begin{abstract}
The vision for precision medicine is to use individual patient 
characteristics to inform a personalized treatment plan that 
leads to the best possible healthcare for each patient.  Mobile 
technologies have an important role to play in this vision as they 
offer a means to monitor a patient's health status in real-time 
and subsequently to deliver interventions if, when, and in the 
dose that they are needed.   Dynamic treatment regimes 
formalize individualized treatment plans as sequences of 
decision rules, one per stage of clinical intervention, that 
map current patient information to a recommended treatment. 
However, most existing methods for estimating optimal dynamic treatment 
regimes are designed for a small number of fixed decision points occurring 
on a coarse time-scale.  We propose a new reinforcement learning 
method for estimating an optimal treatment regime that is applicable 
to data collected using mobile technologies in an outpatient setting. 
The proposed method accommodates an indefinite time horizon and 
minute-by-minute decision making that are common in mobile health 
applications.  We show that the proposed estimators are consistent 
and asymptotically normal under mild conditions.  The proposed 
methods are applied to estimate an optimal dynamic treatment regime 
for controlling blood glucose levels in patients with type 1 diabetes.
\end{abstract}

\noindent%
\vfill

\newpage
\spacingset{1} 

\section{Introduction} \label{vl.intro}
The use of mobile devices in clinical care, called mobile health (mHealth), provides an 
effective and scalable platform to assist patients in managing their illness 
\citep[][]{free2013effectiveness, steinhubl2013can}. 
Advantages of mHealth interventions include real-time communication between a 
patient and their health-care provider as well as systems for delivering 
training, teaching, and social support \citep[][]{kumar2013mobile}. 
Mobile technologies can also be used to collect rich longitudinal data 
to estimate optimal dynamic treatment regimes and to deliver treatment that 
is deeply tailored to each individual patient. 
We propose a new estimator of an optimal treatment regime that is 
suitable for use with with longitudinal data collected in mHealth applications.

A dynamic treatment regime provides a framework to administer individualized treatment over time through 
a series of decision rules. Dynamic treatment regimes 
have been well-studied in the statistical and biomedical literature 
\citep[][]{murphy2003optimal, robins2004optimal, moodie2007demystifying, kosorok2015adaptive, chakraborty2013statistical} 
and furthermore, statistical considerations in mHealth have been studied by, for example, \cite{liao2015sample} 
and \cite{klasnja2015microrandomized}. 
Although mobile technology has been successfully utilized in clinical areas such as 
diabetes \citep[][]{quinn2011cluster, maahs2012outpatient}, 
smoking cessation \citep[][]{ali2012mpuff}, and 
obesity \citep[][]{bexelius2010measures}, 
mHealth poses some unique challenges that preclude direct application of 
existing methodologies for dynamic treatment regimes. For example, mHealth  
applications typically involve a large number of time points per individual and no definite time horizon; 
the momentary signal may be weak and may not directly measure the outcome of interest; and estimation 
of optimal treatment strategies must be done online as data accumulate. 

This work is motivated in part by our involvement in a study of mHealth as 
a management tool for type 1 diabetes. 
Type 1 diabetes is an autoimmune disease wherein the pancreas produces insufficient levels of insulin, 
a hormone needed to regulate blood glucose concentration.
Patients with type 1 diabetes are continually engaged in management activities including
monitoring glucose levels, timing and dosing insulin injections, and 
regulating diet and physical activity. 
Increased glucose monitoring and attention to self-management 
facilitate more frequent treatment adjustments and have been shown 
to improve patient outcomes \citep[][]{levine2001predictors, haller2004predictors, ziegler2011frequency}. 
Thus, patient outcomes have the potential to 
be improved by diabetes management tools which are deeply tailored to the 
continually evolving health status of each patient. 
Mobile technologies can be used to collect data on physical activity, glucose, and insulin at a fine granularity
in an outpatient setting \citep[][]{maahs2012outpatient}. 
There is great potential for using these data to create comprehensive and accessible 
mHealth interventions for clinical use. 
We envision application of this work for use before 
the artificial pancreas \citep[][]{weinzimer2008fully, kowalski2015pathway, bergenstal2016safety} 
becomes widely available. 

The sequential decision making process can be modeled as a Markov decision process 
\citep[][]{puterman2014markov} and the optimal treatment regime can 
be estimated using reinforcement learning algorithms such as Q-learning 
\citep[][]{murphy2005generalization, zhao2009reinforcement, tang2012developing, schulteOLD}. 
\cite{ertefaie2014constructing} proposed a variant of greedy gradient Q-learning (GGQ) 
to estimate optimal dynamic treatment regimes in infinite horizon 
settings \citep[see also][]{maei2010toward}. 
In GGQ, the form of the estimated Q-function dictates the 
form of the estimated optimal treatment regime. Thus, one must choose between 
a parsimonious model for the Q-function at the risk of model misspecification 
or a complex Q-function that yields unintelligible treatment regimes.  
Furthermore, GGQ requires modeling a non-smooth function of the data, which creates complications
\citep[][]{laber2014interactive, linn2017interactive}. 
We propose an alternative estimation method for infinite horizon dynamic treatment regimes 
that is suited to mHealth applications. Our approach, which we call V-learning, 
involves estimating the optimal policy among a prespecified class of policies 
\citep[][]{zhang2012robust, zhang2013robust}. 
It requires minimal assumptions about the data-generating process 
and permits estimating a randomized decision rule that can be implemented online 
as data accumulate. 

In Section~\ref{offline}, we describe the setup and present our method 
for offline estimation using data from a micro-randomized trial or observational study. 
In Section~\ref{online}, we extend our method for application to online estimation with accumulating data. 
Theoretical results, including consistency and asymptotic normality of the proposed estimators, 
are presented in Section~\ref{vl.theory}. We compare the proposed 
method to GGQ using simulated data in Section~\ref{vl.sim}. 
A case study using data from patients 
with type 1 diabetes is presented in Section~\ref{vl.data} 
and we conclude with a discussion in Section~\ref{vl.conc}. 
Proofs of technical results are in the Appendix. 

\section{Offline estimation from observational data} \label{offline}
We assume that the available data are 
$\left\lbrace \left( \bS_i^1, A_i^1, \bS_i^2,\ldots, \bS_{i}^{T_i}, 
    A_{i}^{T_i}, \bS_{i}^{T_i+1} \right) \right\rbrace_{i=1}^{n}$, 
which comprise $n$ independent, identically distributed trajectories 
$\left(\bS^1, A^1, \bS^2, \ldots, \bS^T, A^T, \bS^{T+1}\right)$, 
where: $\bS^t \in \mathbb{R}^p$ denotes a 
summary of patient information collected up to and including time $t$; 
$A^t \in \mathcal{A}$ denotes the treatment assigned at
time $t$; and $T \in \mathbb{Z}_{+}$ denotes the (possibly random)
patient follow-up time. In the motivating example of type 1 diabetes, 
$\bS^t$ could contain a patient's blood glucose, dietary intake, and 
physical activity in the hour leading up to time $t$ and $A^t$ could denote 
an indicator that an insulin injection is taken at time $t$. 
We assume that the data-generating model is a
time-homogeneous Markov process so that
$\bS^{t+1} \independent (A^{t-1}, \bS^{t-1},\ldots, A^1, \bS^1) \big| (A^t,
\bS^t)$ and the conditional density $p(\bs^{t+1} | a^t, \bs^t)$ 
is the same for all $t \ge 1$. 
Let $L^{t}\in \lbrace 0, 1\rbrace$ denote 
an indicator that the patient is still in follow-up at time $t$, i.e.,
$L^t = 1$ if the patient is being followed at time $t$ and zero
otherwise.   We assume that $L^t$ is contained in $\bS^t$ so that
$P(L^{t+1} = 1|A^t, \bS^t,\ldots, A^1, \bS^1) = P(L^{t+1}=1 | A^t,
\bS^t)$ and $L^t = 0$ implies $L^{t+1} = 0$ with probability one.  
Furthermore, we assume a known utility function
$u:\mathbb{R}^p \times \mathcal{A}\times \mathbb{R}^p\rightarrow
\mathbb{R}$
so that $U^t = u(\bS^{t+1}, A^{t} , \bS^{t})$ measures the
`goodness' of
choosing treatment $A^t$ in state $\bS^t$ and subsequently
transitioning to state $\bS^{t+1}$. In our motivating example, 
the utility at time $t$ could be a measure of 
how far the patient's average blood glucose concentration deviates 
from the optimal range over the hour preceding and following time $t$. 
The goal is to select treatments
to maximize expected cumulative utility; treatment selection is
formalized using a treatment regime \citep[][]{schulteOLD, kosorok2015adaptive} 
and the utility associated with
any regime is defined using potential outcomes \citep[][]{rubin}.

Let $\mathcal{B}(\mathcal{A})$ denote the space of probability
distributions over $\mathcal{A}$.  A treatment regime in this context
is a function
$\pi:\,\mathrm{dom}\,\bS^t\rightarrow \mathcal{B}(\mathcal{A})$ so that, under
$\pi$, a decision maker presented with state $\bS^t=\bs^t$ at time $t$
will select action $a^t\in\mathcal{A}$ with probability
$\pi(a^t;\bs^t)$.  
Define $\overline{a}^{t} = (a^1,\ldots, a^t)\in\mathcal{A}^t$,  and
$\overline{a}^{\infty} = (a^1, a^2,\ldots)\in\mathcal{A}^{\infty}$.
The set of potential outcomes is
\begin{multline*}
\bW^* =  \Big\lbrace
\bS^1, \bS^{*2}(a^1), \ldots, 
\bS^{*T^*(\overline{a}^{\infty})}(\overline{a}^{T^*(\overline{a}^{\infty}) - 1})
\,:\, \\ T^*(\overline{a}^{\infty}) 
= \inf\left\lbrace t \ge 1\,:\, L^{*t}(\overline{a}^{t-1}) = 0
\right\rbrace,\,\overline{a}^{\infty}\in \mathcal{A}^{\infty}
\Big\rbrace,
\end{multline*}
where $\bS^{*t}(\overline{a}^{t-1})$ is the potential state and
$L^{*t}(\overline{a}^{t-1})$ is the potential follow-up status at time
$t$ under treatment sequence $\overline{a}^{t-1}$. Thus, the
potential utility at time $t$ is $U^{*t}(\overline{a}^t) = 
u\left\lbrace
\bS^{*(t+1)}(\overline{a}^t), a^t, \bS^{*t}(\overline{a}^{t-1})
\right\rbrace$.  For any $\pi$, define $\left\lbrace \xi_{\pi}^t(\cdot)
\right\rbrace_{t\ge 1}$ to be a sequence of independent, $\mathcal{A}$-valued 
 stochastic processes indexed by $\mathrm{dom}\,\bS^t$ such that 
$P\left\lbrace \xi_{\pi}^t(\bs^t) = a^t\right\rbrace = \pi(a^t;\bs^t)$.  
The potential follow-up time under $\pi$ is
\begin{equation*}
T^*(\pi) = \sum_{t \ge 1} \sum_{\overline{a}^t \in \mathcal{A}^t} 
t 1\left\{ \mathrm{sup}_{\underline{a}^{t+1}} T^*(\overline{a}^t, \underline{a}^{t+1}) = t \right\}
\prod_{v=1}^t 1\left[\xi_\pi^{v}\left\lbrace\bS^{*v}(\overline{a}^{v-1})
\right\rbrace = a^v\right],
\end{equation*}
where $\underline{a}^{t+1} = (a^{t+1}, a^{t+2}, \ldots)$. 
The potential utility under $\pi$ at time $t$ is
\begin{equation*}
U^{*t}(\pi) =
\left\lbrace
\begin{array}{ll}
\sum_{\overline{a}^{t}\in\mathcal{A}^t}
U^{*t}\left(\overline{a}^t\right)\prod_{v=1}^t
1\left[
\xi_{\pi}^{v}\left\lbrace
\bS^{*v}(\overline{a}^{v-1})
\right\rbrace = a^v
\right], & \,\,\,\,\mathrm{if}\,\,\,\,T^*(\pi) \ge t \\
0,       & \,\,\,\,\mathrm{otherwise}, 
\end{array}
\right.
\end{equation*}  
where $\bS^{*1}(\overline{a}^0) = \bS^1$. 
Thus, utility is set to zero after a patient is lost to follow-up. 
However, in certain situations, utility may be constructed so as to 
take a negative value at the time point when the patient is lost to follow-up, 
e.g., if the patient discontinues treatment because of a negative effect 
associated with the intervention. 
Define the state-value function
$V(\pi, \bs^t) = \mathbb{E}\left\{ \sum_{k \ge
    0}\gamma^{k}U^{*(t+k)}(\pi)\big|\bS^t=\bs^t \right\}$
\citep[][]{sutton},
where $\gamma\in (0,1)$ is a fixed constant that captures the
trade-off between short- and long-term outcomes.  For any distribution
$\mathcal{R}$ on $\mathrm{dom}\,\bS^1$, 
define the value function with
respect to reference distribution $\mathcal{R}$ as
$V_{\mathcal{R}}(\pi) = \int V(\pi, \bs)\mathrm{d}\mathcal{R}(\bs)$;
throughout, we assume that this reference distribution is fixed. 
The reference distribution can be thought of as a distribution 
of initial states and we estimate it from the data in the 
implementation in Sections~\ref{vl.sim} and \ref{vl.data}. 
For a prespecified class of regimes, $\Pi$, the optimal
regime, $ \pi_{\mathcal{R}}^{\mathrm{opt}} \in \Pi$, satisfies
$V_{\mathcal{R}}(\pi_{\mathcal{R}}^{\mathrm{opt}}) \ge
V_{\mathcal{R}}(\pi)$ for all $\pi\in\Pi$.

To construct an estimator of $\pi_{\mathcal{R}}^{\mathrm{opt}}$, we
make a series of assumptions that connect the potential outcomes 
in $\bW^*$  with the data-generating model. 
\begin{assumption} \label{ignorability.assume}
Strong ignorability, $A^t \independent \bW^* \big| \bS^t$ for all $t$.
\end{assumption}
\begin{assumption} \label{consistency.assume}
Consistency, $\bS^t = \bS^{*t}(\overline{A}^{t-1})$ for all $t$ and
$T = T^*(\overline{A}^{\infty})$.
\end{assumption}
\begin{assumption} \label{positivity.assume}
Positivity, there
exists $c_0 > 0$ so that $P(A^t=a^t|\bS^t=\bs^t) \ge c_0$
for all $a^t\in\mathcal{A}$, $\bs^t \in \mathrm{dom}\,\bS^t$, and
all $t$.
\end{assumption} 
In addition, we implicitly assume that there is no interference among 
the experimental units.  These assumptions are common in the
context of estimating dynamic treatment regimes
\citep[][]{robins2004optimal, schulteOLD}. 
Assumptions~\ref{ignorability.assume} and \ref{positivity.assume} hold by construction in a 
micro-randomized trial \citep{klasnja2015microrandomized, liao2015sample}. 

Let $\mu^t(a^t; \bs^t) = P(A^t=a^t|\bS^t=\bs^t)$ for each $t\ge 1$. 
In a micro-randomized trial, $\mu^t(a^t; \bs^t)$ is a known 
randomization probability; in an observational study, it 
must be estimated from the data. 
The following lemma characterizes $V_{\mathcal{R}}(\pi)$ for any 
regime, $\pi$, in terms of the data-generating model 
\citep[see also Lemma~4.1 of][]{murphy2001marginal}.  
A proof is provided in the appendix.
\begin{lem}\label{lemma}
Let $\pi$ denote an arbitrary regime and  $\gamma \in
(0,1)$ a discount factor. Then, under 
assumptions~\ref{ignorability.assume}-\ref{positivity.assume} and 
provided interchange of the sum and integration 
is justified, the state-value function of $\pi$ at $\bs^t$ is 
\begin{equation}\label{vRPiIPW}
V(\pi, \bs^t) = \sum_{k\ge 0}\mathbb{E}\left[
\gamma^{k}U^{t+k}\left\lbrace
\prod_{v=0}^k\frac{\pi(A^{v+t}; \bS^{v+t})}{\mu^{v+t}(A^{v+t}; \bS^{v+t})}
\right\rbrace \bigg|\bS^{t}=\bs^t
\right].
\end{equation}
\end{lem}\noindent
The preceding result will form the basis for an estimating equation
for  $V_{\mathcal{R}}(\pi)$.  Write the right hand side of 
(\ref{vRPiIPW}) as
\begin{eqnarray*}
V(\pi, \bS^t) & = & \mathbb{E}\left\{
\frac{\pi(A^t;\bS^t)}{\mu^{t}(A^t;\bS^t)}
\left( 
U^{t} + \gamma
\sum_{k\ge 0}
\mathbb{E}\left[
\gamma^{k}U^{t+k+1}
\left\lbrace 
\prod_{v=0}^{k}\frac{\pi(A^{v+t+1}; \bS^{v+t+1})}{\mu^{v+t+1}(A^{v+t+1}; \bS^{v+t+1})}
\right\rbrace 
\bigg|\bS^{t+1}
\right]
\right)
\bigg|\bS^t
\right\}\\ 
& = & \mathbb{E}\left[
\frac{\pi(A^t;\bS^t)}{\mu^{t}(A^t;\bS^t)}\left\lbrace
U^t + \gamma V(\pi, \bS^{t+1})
\right\rbrace\bigg|\bS^t
\right], 
\end{eqnarray*}
from which it follows that
\begin{equation*}
0 = \mathbb{E}\left[
\frac{\pi(A^t;\bS^t)}{\mu^{t}(A^t;\bS^t)}\left\lbrace
U^t + \gamma V(\pi, \bS^{t+1}) - V(\pi, \bS^t)
\right\rbrace\bigg|\bS^t
\right].
\end{equation*}
Subsequently, for any function $\psi$ defined on
$\mathrm{dom}\,\bS^t$, the state-value function satisfies 
\begin{equation}\label{popnEstEqn}
0 = \mathbb{E}\left[
\frac{\pi(A^t;\bS^t)}{\mu^{t}(A^t;\bS^t)}\left\lbrace
U^t + \gamma V(\pi, \bS^{t+1}) - V(\pi, \bS^t)
\right\rbrace \psi(\bS^t)\right],
\end{equation}
which is an importance-weighted variant of the well-known Bellman
optimality equation \citep[][]{sutton}.  

Let $V(\pi, \bs; \theta^{\pi})$ denote a model for $V(\pi, \bs)$
indexed by $\theta^{\pi}\in \Theta \subseteq \mathbb{R}^q$.  We assume
that the map $\theta^{\pi} \mapsto  V(\pi, \bs;\theta^{\pi})$ is
differentiable everywhere for each fixed $\bs$ and $\pi$.  
Let $\nabla_{\theta^{\pi}} V(\pi, \bs; \theta^{\pi})$ denote the
gradient of
$V(\pi, \bs; \theta^{\pi})$ and define
\begin{equation} \label{est.eq}
\Lambda_{n}(\pi, \theta^{\pi}) = 
\frac{1}{n}\sum_{i=1}^{n}\sum_{t=1}^{T_i}
\frac{\pi(A_i^t;\bS_i^t)}{\mu^{t}(A_i^t;\bS_i^t)}\left\lbrace
U_i^t + \gamma V(\pi, \bS_i^{t+1};\theta^{\pi}) - V(\pi, \bS_i^t;\theta^{\pi})
\right\rbrace 
\nabla_{\theta^{\pi}}V(\pi, \bS_i^t; \theta^{\pi}).
\end{equation}
Given a positive definite matrix  
$\Omega \in \mathbb{R}^{q \times q}$ 
and penalty function 
$\mathcal{P}: \mathbb{R}^q \rightarrow \mathbb{R}_+$, 
define
$\widehat{\theta}_{n}^{\pi} =
\arg\min_{\theta^{\pi}\in\Theta}\left\lbrace \Lambda_{n}(\pi,
\theta^{\pi})^{\T}\Omega
\Lambda_{n}(\pi, \theta^{\pi}) + \lambda_n 
\mathcal{P}(\theta^{\pi})\right\rbrace$, 
where $\lambda_n$ is a tuning parameter. 
Subsequently, $V \left(\pi, \bs;
\widehat{\theta}_{n}^{\pi} \right)$ is the estimated state-value function
under $\pi$ in state $\bs$.  Thus, given a reference distribution, 
$\mathcal{R}$, the estimated value of a regime, $\pi$, 
is $\widehat{V}_{n,\mathcal{R}}(\pi) = 
\int V \left(\pi, \bs; \widehat{\theta}_{n}^{\pi} \right)\mathrm{d}\mathcal{R}(\bs)$ and the 
estimated optimal regime is $\widehat{\pi}_{n} = 
\arg\max_{\pi\in\Pi}\widehat{V}_{n, \mathcal{R}}(\pi)$. 
The idea of V-learning is to use estimating equation~(\ref{est.eq}) 
to estimate the value of any policy and maximize estimated value 
over a class of policies; we will discuss strategies for this 
maximization in Section~\ref{vl.sim}. 

V-learning requires a parametric class of policies. 
Assuming that there are $K$ possible treatments, $a_1, \ldots, a_K$, we can define 
a parametric class of policies as follows. Define 
$
\pi(a_j; \bs, \beta) = \mathrm{exp}(\bs^\T \beta_j)/\left\{1 + \sum_{k = 1}^{K - 1} \mathrm{exp}(\bs^\T \beta_k) \right\}
$
for $j = 1, \ldots, K - 1$, and $\pi(a_K; \bs) = 1 / \left\{1 + \sum_{k = 1}^{K - 1} \mathrm{exp}(\bs^\T \beta_k)\right\}$. 
This defines a class of randomized policies parametrized by 
$\beta = (\beta_1^\T, \ldots, \beta_{K - 1}^\T)^\T$, where 
$\beta_k$ is a vector of parameters for the $k$-th treatment. Under a policy in this class
defined by $\beta$, actions are selected stochastically according to the 
probabilities $\pi(a_j; \bs, \beta)$, $j = 1, \ldots, K$. 
In the case of a binary treatment, a policy in this class reduces to 
$\pi(1; \bs, \beta) = \mathrm{exp}(\bs^\T \beta) / \left\{1 + \mathrm{exp}(\bs^\T \beta) \right\}$ 
and $\pi(0; \bs, \beta) = 1 / \left\{1 + \mathrm{exp}(\bs^\T \beta) \right\}$ 
for a $p \times 1$ vector $\beta$. This class of policies is used in the 
implementation in Sections~\ref{vl.sim} and \ref{vl.data}. 

V-learning also requires a class of models for the state value function 
indexed by a parameter, $\theta^\pi$. 
We use a basis function approximation. 
Let $\Phi = (\phi_1, \ldots, \phi_q)^\T$ be a vector of prespecified basis functions 
and let $\Phi(\bs_i^t) = \{\phi_1(\bs_i^t), \ldots, \phi_q(\bs_i^t)\}^\T$. 
Let $V(\pi, \bs_i^t; \theta^\pi) = \Phi(\bs_i^t)^\T \theta^\pi$.  
Under this working model, 
\begin{eqnarray} \label{linearV}
\Lambda_n(\pi, \theta^\pi) & = & \left[n^{-1} \sum_{i=1}^n \sum_{t=1}^{T_i}  
 \frac{\pi(A_i^t; \bS_i^t)}{\mu^t(A_i^t; \bS_i^t)} 
 \left\{ \gamma \Phi(\bS_i^t) \Phi(\bS_i^{t+1})^\T - \Phi(\bS_i^t) \Phi(\bS_i^t)^\T \right\} \right] \theta^\pi \\ \nonumber
 & & + n^{-1} \sum_{i=1}^n \sum_{t=1}^{T_i} \left\{\frac{\pi(A_i^t; \bS_i^t)}{\mu^t(A_i^t; \bS_i^t)} 
 U_i^t \Phi(\bS_i^t) \right\}. \nonumber
\end{eqnarray}
Computational efficiency is gained from the linearity of 
$V(\pi, \bs_i^t; \theta^\pi)$ in $\theta^\pi$; 
flexibility can be achieved through the choice of $\Phi$. 
We examine the performance of V-learning using a variety of basis functions 
in Sections~\ref{vl.sim} and \ref{vl.data}. 

\section{Online estimation from accumulating data} \label{online}

Suppose we have accumulating data 
$\left\lbrace \left( \bS_i^1, A_i^1, \bS_i^2, \ldots\right) \right\rbrace_{i=1}^{n}$, 
where $\bS_i^t$ and $A_i^t$ represent the state and action for patient $i=1, \ldots, n$ at time $t \ge 1$. 
At each time $t$, we estimate an optimal policy in a class, $\Pi$, using data collected up 
to time $t$, take actions according to the estimated optimal policy, and 
estimate a new policy using the resulting states. Let $\widehat{\pi}_n^{t}$ 
be the estimated policy at time $t$, i.e., $\widehat{\pi}_n^{t}$ is 
estimated after observing state $\bS^{t+1}$ and before taking action $A^{t+1}$. 
If $\Pi$ is a class of randomized policies, we can select an action for 
a patient presenting with $\bS^{t + 1} = \bs^{t + 1}$ 
according to $\widehat{\pi}_n^{t}(\cdot; \bs^{t + 1})$, i.e., 
we draw $A^{t+1}$ according to the distribution 
$P(A^{t + 1} = a) = \widehat{\pi}_n^{t}(a; \bs^{t + 1})$.
If a class of deterministic policies is of interest, we can inject some randomness into 
$\widehat{\pi}_n^{t}$ to facilitate exploration. 
One way to do this is an $\epsilon$-greedy strategy \citep[][]{sutton}, which 
selects the estimated optimal action with probability $1 - \epsilon$ and otherwise samples 
equally from all other actions. Because an $\epsilon$-greedy strategy can 
be used to introduce randomness into a deterministic policy, 
we can assume a class of randomized policies. 
 
At each time $t \ge 1$, let 
$\widehat{\theta}_{n, t}^{\pi} = \mathrm{arg \, min}_{\theta^\pi \in \Theta} 
\left\{ \Lambda_{n,t}(\pi, \theta^\pi)^\T \Omega \Lambda_{n,t}(\pi, \theta^\pi) 
+ \lambda_n \mathcal{P}(\theta^\pi) \right\}$, where $\Omega$, $\lambda_n$, and $\mathcal{P}$ 
are as defined in Section~\ref{offline} and  
\begin{equation} \label{online.est}
\Lambda_{n,t}(\pi, \theta^\pi) = \frac{1}{n}\sum_{i=1}^{n}\sum_{v=1}^{t}
\frac{\pi(A_i^v;\bS_i^v)}{\widehat{\pi}_n^{v-1}(A_i^v;\bS_i^v)}\left\lbrace
U_i^v + \gamma V(\pi, \bS_i^{v+1};\theta^{\pi}) - V(\pi, \bS_i^v;\theta^{\pi})
\right\rbrace 
\nabla_{\theta^{\pi}}V(\pi, \bS_i^v; \theta^{\pi})
\end{equation}
with $\widehat{\pi}_n^0$ some initial randomized policy. 
We note that estimating equation~(\ref{online.est}) is similar to (\ref{est.eq}), 
except that $\widehat{\pi}_n^{v-1}$ replaces $\mu^v$ as the data-generating policy. 
Given the estimator of the value of $\pi$ at time $t$, 
$\widehat{V}_{n, \mathcal{R}, t}(\pi) = \int V 
\left(\pi, \bs; \widehat{\theta}_{n, t}^\pi \right) \mathrm{d} \mathcal{R}(\bs)$, 
the estimated optimal policy at time $t$ is $\widehat{\pi}_n^{t} = \mathrm{arg \, max}_{\pi \in \Pi} 
\widehat{V}_{n, \mathcal{R}, t}(\pi)$. 
In practice, we may choose to update the policy in batches rather 
than at every time point. An alternative 
way to encourage exploration through the action space is to choose 
$\widehat{\pi}_n^{t} = \mathrm{arg \, max}_{\pi \in \Pi} 
\left\{ \widehat{V}_{n, \mathcal{R}, t}(\pi) + 
\alpha^{t} \widehat{\psi}^{t}(\pi)\right\}$ 
for some sequence $\alpha^t \ge 0$, 
where $\widehat{\psi}^{t}(\pi)$ is a measure of uncertainty in 
$\widehat{V}_{n, \mathcal{R}, t}(\pi)$. An example of this is 
upper confidence bound sampling, or UCB \citep[][]{lai1985asymptotically}. 

It some settings, when the data-generating process may vary across patients, 
it may be desirable to allow each patient to follow 
an individualized policy that is estimated using only that patient's 
data. Suppose that $n$ patients are followed for an initial 
$T_1$ time points after which the policy $\widehat{\pi}_n^1$ is estimated. 
Then, suppose that patient $i$ follows $\widehat{\pi}_n^1$ 
until time $T_2$, when a policy $\widehat{\pi}^2_i$ is estimated using only the 
states and actions observed for patient $i$. This procedure is then 
carried out until time $T_K$ for some fixed $K$ with each 
patient following their own individual policy which is adapted 
to match the individual over time. We may also choose to adapt the randomness 
of the policy at each estimation. For example, we could select 
$\epsilon_1 > \epsilon_2 > \ldots > \epsilon_K$ and, following estimation $k$, 
have patient $i$ follow policy $\widehat{\pi}^k_i$ with probability 
$1 - \epsilon_k$ and policy $\widehat{\pi}_n^1$ with probability $\epsilon_k$. 
In this way, patients become more likely to follow their own individualized 
policy and less likely to follow the initial policy over time, reflecting 
increasing confidence in the individualized policy as more data become available. 
The same class of policies and model for the state value function can be 
used as in Section~\ref{offline}. 

\section{Theoretical results} \label{vl.theory}

In this section, we establish asymptotic properties of $\widehat{\theta}_n^\pi$ 
and $\widehat{\pi}_n$ for offline estimation. Throughout, we assume 
assumptions~\ref{ignorability.assume}-\ref{positivity.assume}  
from Section~\ref{offline}. 

Let $\widehat{\theta}_n^\pi = \mathrm{arg \, min}_{\theta^\pi \in \Theta} 
\left\{ \Lambda_n(\pi, \theta^\pi)^\T \Lambda_n(\pi, \theta^\pi) 
+ \lambda_n (\theta^\pi)^\intercal \theta^\pi \right\}$. 
Thus, we use the squared Euclidean norm of $\theta$ as the penalty function; 
we will assume that $\lambda_n = o_P(n^{-1/2})$. 
For simplicity, we let $\Omega$ be the identity matrix. 
Assume the working model for the state value function 
introduced in Section~\ref{offline}, i.e., 
$V(\pi, \bs_i^t; \theta^\pi) = \Phi(\bs_i^t)^\T \theta^\pi$.
For fixed $\pi$, denote the true $\theta^\pi$ by $\theta^\pi_0$, 
i.e., $V(\pi, \bs) = \Phi(\bs)^\T \theta^\pi_0$. 
Let $\nu = \int \Phi(\bs) \mathrm{d} \mathcal{R}(\bs)$ so that 
$V_\mathcal{R}(\pi) = \nu^\T \theta^\pi_0$. 
Define $\widehat{V}_{n, \widehat{\mathcal{R}}}(\pi) = \left\{\mathbb{E}_n \Phi(\bS) \right\}^\T \widehat{\theta}^\pi_n$, 
where $\mathbb{E}_n$ denotes the empirical measure of the observed data. 
Let $\Pi = \left\{ \pi_\beta : \beta \in \mathcal{B} \right\}$ be a parametric class of policies 
and let $\widehat{\pi}_{n} = \pi_{\widehat{\beta}_n}$ where 
$\widehat{\beta}_n = \arg\max_{\beta \in \mathcal{B}} \widehat{V}_{n, \widehat{\mathcal{R}}}(\pi_\beta)$. 

Our main results are summarized in Theorems~\ref{main.theta} and \ref{main.beta} below. 
Because each patient trajectory is a stationary Markov chain, we need to 
use asymptotic theory based on stationary processes; consequently, some of the 
required technical conditions are more difficult to verify 
than those for i.i.d.\ data. Define the bracketing integral 
for a class of functions, $\mathcal{F}$, 
by $J_{[]}\left\{\delta, \mathcal{F}, L_r(P)\right\} = 
\int_0^\delta \sqrt{\log N_{[]}\left\{\epsilon, \mathcal{F}, L_r(P)\right\}}\mathrm{d}\epsilon$, 
where the bracketing number for $\mathcal{F}$, $N_{[]}\left\{\epsilon, \mathcal{F}, L_r(P)\right\}$, 
is the number of $L_r(P)$ $\epsilon$-brackets 
needed such that each element of $\mathcal{F}$ is contained in at least one bracket 
\citep[see Chapter 2 of][]{kosorok2008introduction}. 
For any stationary sequence of possibly dependent random variables, $\{X^t\}_{t \ge 1}$, 
let $\mathcal{M}_b^c$ be the $\sigma$-field 
generated by $X^b, \ldots, X^c$ and define 
$\zeta(k) = \mathbb{E} \left[ \sup_{m \ge 1} 
\left\{ |P(B | \mathcal{M}_1^m) - P(B)| : B \in \mathcal{M}_{m + k}^\infty\right\} \right]$.
We say that the chain $\{X^t\}_{t \ge 1}$ is absolutely regular if $\zeta(k) \rightarrow 0 $ as $k \rightarrow 0$ 
\citep[also called $\beta$-mixing in Chapter 11 of][]{kosorok2008introduction}.  
We make the following assumptions. 

\begin{assumption} \label{rho.assume}
There exists a $2 < \rho < \infty$ such that 
\begin{enumerate}
\item $\mathbb{E} |U^t|^{3\rho} < \infty$, $\mathbb{E} \| \Phi(\bS^t) \|^{3\rho} < \infty$, 
and $ \mathbb{E} \| \bS^t \|^{3\rho} < \infty$.
\label{moment.assume}
\item The sequence $\{(\bS^t, A^t)\}_{t \ge 1}$ is absolutely regular with 
$\sum_{k = 1}^\infty k^{2 / (\rho - 2)} \zeta(k) < \infty$.
\item The bracketing integral of the class of policies, $J_{[]}\left\{\infty, \Pi, L_{3\rho}(P)\right\} < \infty$. 
\label{entropy.assume}
\end{enumerate}
\end{assumption}

\begin{assumption} \label{invertibility.assume}
There exists some $c_1 > 0$ such that 
$$
\inf_{\pi \in \Pi} c^\T \mathbb{E} \left[ \frac{\pi(A^t; \bS^t)}{\mu^t(A^t; \bS^t)} 
\left\{ \Phi(\bS^t) \Phi(\bS^t)^\T - \gamma^2 \Phi(\bS^{t + 1}) \Phi(\bS^{t + 1})^\T \right\} \right] c 
\ge c_1 \| c \|^2 
$$ 
for all $c \in \mathbb{R}^q$. 
\end{assumption}

\begin{assumption} \label{unique.assume}
The map $\beta \mapsto V_\mathcal{R}(\pi_\beta)$ has a unique and well separated maximum 
over $\beta$ in the interior of $\mathcal{B}$; let $\beta_0$ denote the maximizer.  
\end{assumption}

\begin{assumption} \label{continuity.assume}
The following condition holds: $\sup_{\|\beta_1 - \beta_2\| \le \delta} \mathbb{E} 
\| \pi_{\beta_1}(A; \bS) - \pi_{\beta_2}(A; \bS)\| \rightarrow 0$ as $\delta \downarrow 0$. 
\end{assumption}

\begin{rmrk}
Assumption~\ref{rho.assume} requires certain finite moments and that the dependence between 
observations on the same patient vanishes as observations become further apart. 
In Lemma~\ref{policy.class} in the appendix, we verify part~\ref{entropy.assume} of assumption~\ref{rho.assume} 
and assumption~\ref{continuity.assume} 
for the class of policies introduced in Section~\ref{offline}. 
However, note that the theory holds for any class of policies 
satisfying the given assumptions, not just the class considered here. 
Assumption~\ref{invertibility.assume} is needed to 
show the existence of a unique $\theta_0^\pi$ uniformly over $\Pi$ and 
assumption~\ref{unique.assume} requires that the true optimal decision in each state 
is unique \citep[see assumption A.8 of][]{ertefaie2014constructing}. 
Assumption~\ref{continuity.assume} requires smoothness on the class of policies. 
\end{rmrk}

The main results of this section are stated below. 
Theorem~\ref{main.theta} states that there exists a unique solution to $0 = \mathbb{E}\Lambda_n(\pi, \theta^\pi)$ 
uniformly over $\Pi$ and that the estimator $\widehat{\theta}_n$ converges weakly to a mean zero 
Gaussian process in $\ell^\infty(\Pi)$. 
\begin{thm} \label{main.theta}
Under the given assumptions, the following hold. 
\begin{enumerate}
\item For all $\pi \in \Pi$, there exists a $\theta_0^\pi \in \mathbb{R}^q$ such that 
$\mathbb{E}\Lambda_n(\pi, \theta^\pi)$ has a zero at $\theta^\pi = \theta^\pi_0$. 
Moreover, $\sup_{\pi \in \Pi} \| \theta_0^\pi\| < \infty$ and 
$\sup_{\| \beta_1 - \beta_2 \| \le \delta} \left\| \theta^{\pi_{\beta_1}}_0 - \theta^{\pi_{\beta_2}}_0 \right\| \rightarrow 0$ 
as $\delta \downarrow 0$.
\label{main1} 
\item Let $\mathbb{G}(\pi)$ be a tight, mean zero Gaussian process indexed by $\Pi$ with covariance 
$\mathbb{E}\left\{ \mathbb{G}(\pi_1) \mathbb{G}(\pi_2) \right\} = w_1(\pi_1)^{-1} w_0(\pi_1, \pi_2) w_1(\pi_2)^{-\T}$ 
where 
\begin{eqnarray*}
w_0(\pi_1, \pi_2) & = & \mathbb{E} \bigg[ \frac{\pi_1(A^t; \bS^t) \pi_2(A^t; \bS^t)}{\mu^t(A^t; \bS^t)^2} 
\left\{ U^t + \gamma \Phi(\bS^{t + 1}) \theta_0^{\pi_1} - \Phi(\bS^t) \theta_0^{\pi_1} \right\} \\
 & & \left\{ U^t + \gamma \Phi(\bS^{t + 1}) \theta_0^{\pi_2} - \Phi(\bS^t) \theta_0^{\pi_2} \right\}
\Phi(\bS^t) \Phi(\bS^t)^\T \bigg]
\end{eqnarray*}
and 
$$
w_1(\pi) = \mathbb{E}\left[ \frac{\pi(A^t; \bS^t)}{\mu^t(A^t; \bS^t)} 
\Phi(\bS^t) \left\{ \Phi(\bS^t) - \gamma \Phi(\bS^{t + 1})\right\}^\intercal \right]. 
$$
Then, $\sqrt{n}\left( \widehat{\theta}_n^\pi - \theta_0^\pi \right) \rightsquigarrow \mathbb{G}(\pi)$ in $\ell^\infty(\Pi)$. 
\label{main2}
\item Let $\mathbb{G}(\pi)$ be as defined in part~\ref{main2}. 
Then, $\sqrt{n} \left\{ \widehat{V}_{n, \widehat{\mathcal{R}}}(\pi) - V_\mathcal{R}(\pi)\right\} 
\rightsquigarrow \nu^\T \mathbb{G}(\pi)$ in $\ell^\infty(\Pi)$. 
\label{main3}
\end{enumerate}
\end{thm}
Theorem~\ref{main.beta} below gives us that the estimated optimal policy converges in probability to the 
true optimal policy over $\Pi$ and that the estimated value of the estimated optimal 
policy converges to the true value of the estimated optimal policy. 
\begin{thm} \label{main.beta}
Under the given assumptions, the following hold. 
\begin{enumerate}
\item Let $\widehat{\beta}_n = \mathrm{arg \, max}_{\beta \in \mathcal{B}} \widehat{V}_{n, \widehat{\mathcal{R}}}(\pi_\beta)$ 
and $\beta_0 = \mathrm{arg \, max}_{\beta \in \mathcal{B}} V_\mathcal{R}(\pi_\beta)$. 
Then, $\left\| \widehat{\beta}_n - \beta_0 \right\| \xrightarrow[]{P} 0$. 
\label{main4}
\item Let $\sigma_0^2 = \nu^\T w_1(\pi_{\beta_0})^{-1} w_0(\pi_{\beta_0}, \pi_{\beta_0}) w_1(\pi_{\beta_0})^{-\T} \nu$. 
Then, $\sqrt{n} \left\{ \widehat{V}_{n, \widehat{\mathcal{R}}}\left(\pi_{\widehat{\beta}_n}\right) 
- V_\mathcal{R}\left(\pi_{\widehat{\beta}_n}\right) \right\} 
\rightsquigarrow N(0, \sigma_0^2)$. 
\label{main5}
\item A consistent estimator for $\sigma_0^2$ is 
$$
\widehat{\sigma}_n^2 = \left\{ \mathbb{E}_n \Phi(\bS^t)\right\}^\T 
\widehat{w}_1\left(\pi_{\widehat{\beta}_n}\right)^{-1} \widehat{w}_0\left(\pi_{\widehat{\beta}_n}, \pi_{\widehat{\beta}_n}\right) 
\widehat{w}_1\left(\pi_{\widehat{\beta}_n}\right)^{-\T} \left\{ \mathbb{E}_n \Phi(\bS^t)\right\}, 
$$
where 
\begin{eqnarray*}
\widehat{w}_0(\pi_1, \pi_2) & = & \mathbb{E}_n \bigg[ \frac{\pi_1(A^t; \bS^t) \pi_2(A^t; \bS^t)}{\mu^t(A^t; \bS^t)^2} 
\left\{ U^t + \gamma \Phi(\bS^{t + 1}) \widehat{\theta}_n^{\pi_1} - \Phi(\bS^t) \widehat{\theta}_n^{\pi_1} \right\} \\
 & & \left\{ U^t + \gamma \Phi(\bS^{t + 1}) \widehat{\theta}_n^{\pi_2} - \Phi(\bS^t) \widehat{\theta}_n^{\pi_2} \right\}
\Phi(\bS^t) \Phi(\bS^t)^\T \bigg] 
\end{eqnarray*}
and
$$
\widehat{w}_1(\pi) = \mathbb{E}_n \left[ \frac{\pi(A^t; \bS^t)}{\mu^t(A^t; \bS^t)} \Phi(\bS^t) 
\left\{ \Phi(\bS^t) - \gamma \Phi(\bS^{t + 1})\right\}^\intercal \right]. 
$$  
\label{main6}
\end{enumerate}
\end{thm}
Proofs of the above results are in the Appendix along with a result on bracketing entropy that is 
needed for the proof of Theorem~\ref{main.theta} and a proof that the class of policies introduced 
above satisfies the necessary bracketing integral assumption. 

\section{Simulation experiments} \label{vl.sim}

In this section, we examine the performance of V-learning 
on simulated data. Section~\ref{sim.offline} contains results for 
offline estimation and Section~\ref{sim.online} contains results 
for online estimation. We begin by discussing an existing method 
for infinite horizon dynamic treatment regimes in Section~\ref{ggq}

\subsection{Greedy gradient Q-learning} \label{ggq}

\cite{ertefaie2014constructing} introduced greedy gradient Q-learning (GGQ) 
for estimating dynamic treatment regimes in infinite horizon settings 
\citep[see also][]{maei2010toward, murphy2016batch}. Here we briefly 
discuss this method. 

Define 
$Q^\pi(\bs^t, a^t) = \mathbb{E}\left\{ \sum_{k \ge 0} \gamma^{k} U^{t + k}(\pi) \Big| \bS^t = \bs^t, A^t = a^t \right\}$. 
The Bellman optimality equation \citep[][]{sutton} is 
\begin{equation} \label{ggq.bellman}
Q^\mathrm{opt}(\bs^t, a^t) = \mathbb{E} \left\{ U^{t} + \gamma \operatorname*{max}_{a \in \mathcal{A}} 
 Q^\mathrm{opt} (\bS^{t + 1}, a) \Big| \bS^t = \bs^t, A^t = a^t \right\}.
\end{equation}
Let $Q(\bs, a; \eta^\mathrm{opt})$ be a parametric model for $Q^\mathrm{opt}(\bs, a)$ indexed 
by $\eta^\mathrm{opt} \in H \subseteq \mathbb{R}^q$. 
In our implementation, we model $Q(\bs, a; \eta^\mathrm{opt})$ as a linear function 
with interactions between all state variables and treatment. 
The Bellman optimality equation motivates the estimating equation 
\begin{equation} \label{ggq.est.eq}
D_n\left(\eta^\mathrm{opt}\right) = \frac{1}{n} \sum_{i = 1}^n \sum_{t = 1}^{T_i} 
\left\{ U_i^{t} + \gamma \operatorname*{max}_{a \in \mathcal{A}} 
 Q(\bS_i^{t + 1}, a; \eta^\mathrm{opt}) - Q(\bS_i^t, A_i^t; \eta^\mathrm{opt}) \right\}
 \nabla_{\eta^\mathrm{opt}} Q(\bS_i^t, A_i^t; \eta^\mathrm{opt}).
\end{equation}
For a positive definite matrix, $\Omega$, we estimate $\eta^\mathrm{opt}$ using 
$\widehat{\eta}^\mathrm{opt}_n = 
\mathrm{arg \, min}_{\eta \in H} D_n(\eta)^\T \Omega D_n(\eta)$. 
The estimated optimal policy in state $\bs$ selects action 
$\widehat{\pi}_n(\bs) = \operatorname*{max}_{a \in \mathcal{A}} Q(\bs, a; \widehat{\eta}^\mathrm{opt}_n)$. 
This optimization problem is non-convex and non-differentiable in $\eta^\mathrm{opt}$. 
However, it can be solved with a generalization of the greedy gradient 
Q-learning algorithm of \cite{maei2010toward}, and hence is referred to as GGQ by 
\cite{ertefaie2014constructing} and in the following. 

The performance of GGQ has been demonstrated in the context of chronic 
diseases with large sample sizes and a moderate number of time points. 
However, in mHealth applications, it is common to have small sample sizes 
and a large number of time points, with decisions occurring at a fine granularity. 
In GGQ, the estimated policy depends directly on $\widehat{\eta}^\mathrm{opt}_n$ and, 
therefore, depends on modeling the transition probabilities of the 
data-generating process. Furthermore, estimating equation~(\ref{ggq.est.eq}) 
contains a non-smooth $\mathrm{max}$ operator, which makes estimation difficult 
without large amounts of data \citep[][]{laber2014interactive, linn2017interactive}. 
V-learning only requires modeling the 
policy and the value function rather than the data-generating process and 
directly maximizes estimated value over a class of policies, thereby 
avoiding the non-smooth $\mathrm{max}$ operator in the estimating 
equation (compare equations (\ref{est.eq}) and (\ref{ggq.est.eq})); 
these attributes may prove advantageous in mHealth settings.   

\subsection{Offline simulations} \label{sim.offline}

Our implementation of V-learning follows the setup in Section~\ref{offline}. 
Maximizing $\widehat{V}_{n, \mathcal{R}}(\pi)$ is done using a combination of simulated annealing and 
the BFGS algorithm as implemented in the optim function in R software \citep[][]{rcoreteam}. 
We note that $\widehat{V}_{n, \mathcal{R}}(\pi)$ is differentiable in $\pi$, thereby avoiding some of the 
computational complexity of GGQ. However, the objective is not necessarily convex. In order to avoid 
local maxima, simulated annealing with $1000$ function evaluations 
is used to find a neighborhood of the maximum; 
this solution is then used as the starting value for the BFGS algorithm. 

We use the class of policies introduced in Section~\ref{offline}. 
Although we maximize the value over a class of randomized policies, the 
true optimal policy is deterministic. To prevent the coefficients of 
$\widehat{\beta}_n$ from diverging to infinity, we add an L2 penalty 
when maximizing over $\beta$. To prevent overfitting, we 
use an L2 penalty when computing $\widehat{\theta}_n^\pi$, i.e.,
$\mathcal{P}(\theta^\pi) = (\theta^\pi)^\intercal \theta^\pi$. Tuning parameters can be used 
to control the amount of randomness in the estimated policy. For example, increasing the 
penalty when computing $\widehat{\beta}_n$ is one way to encourage exploration
through the action space because
$\beta = 0$ defines a policy where each action is selected with equal probability.  

We consider three different models for the state-value function: 
(i) linear; (ii) second degree polynomial; and (iii) Gaussian radial basis functions (RBF).  
The Gaussian RBF is $\phi(x; \kappa, \tau^2) = \mathrm{exp}\left\{-(x - \kappa)^2 / 2 \tau^2\right\}$. 
We use $\tau = 0.25$ and $\kappa = 0, 0.25, 0.5, 0.75, 1$ to create a basis 
of functions and apply this basis to the state variables after scaling them 
to be between 0 and 1. Each model also implicitly contains an intercept. 

We begin with the following simple generative model. 
Let the two-dimensional state vector be $\bS_i^t = (S_{i, 1}^t, S_{i, 2}^t)^\intercal$, 
$i = 1, \ldots, n, t = 1, \ldots, T$. 
We initiate the state variables as independent standard normal random variables 
and let them evolve according to 
$S_{i, 1}^t = (3/4) (2 A_i^{t-1} - 1) S_{i, 1}^{t-1} + (1/4) S_{i, 1}^{t-1} S_{i, 2}^{t-1} + \epsilon_1^t$ 
and $S_{i, 2}^t = (3/4) (1 - 2 A_i^{t-1}) S_{i, 2}^{t-1} + (1/4) S_{i, 1}^{t-1} S_{i, 2}^{t-1} + \epsilon_2^t$, 
where $A_i^t$ takes values in $\{0, 1\}$ and $\epsilon_1^t$ and $\epsilon_2^t$ are independent 
$N(0, 1/4)$ random variables. Define the utility function by 
$U_i^t = u(\bS_i^{t + 1}, A_i^t, \bS_i^t) = 2 S_{i, 1}^{t + 1} + S_{i, 2}^{t + 1} - (1/4)(2 A_i^t - 1)$. 
At each time $t$, we must make a decision to treat or not with 
the goal of maximizing the components of $\bS$ while treating as few times as possible. 
Treatment has a positive effect on $S_1$ and a negative effect on $S_2$. 
We generate $A_i^t$ from a Bernoulli distribution with mean $1/2$. 
In estimation, we assume that the generating model for treatment is known, as 
would be the case in a micro-randomized trial. 

We generate samples of $n$ patients with $T$ time points per patient from the 
given generative model after an initial burn-in period of 50 time points. 
The burn-in period ensures that our simulated data is 
sampled from an approximately stationary distribution. 
We estimate policies using V-learning with three different types of basis functions 
and GGQ. After estimating optimal policies,
we simulate 100 patients following each estimated policy for 100 time points 
and take the mean utility under each policy as an estimate of the value of that policy. 
Estimated values are found in Table~\ref{vl.sim1.1} with Monte Carlo 
standard errors along with observed value. Recall that larger values are better. 
\begin{table}[h!]
\centering
\begin{tabular}{cc|ccccc}
   \hline
$n$ & $T$ & Linear VL & Polynomial VL & Gaussian VL & GGQ & Observed \\ 
   \hline
25 & 24 & $0.118 \, (0.0892)$ & $0.091 \, (0.0825)$ & $0.110 \, (0.0979)$ & $0.014 \, (0.0311)$ & $-0.005$ \\ 
   & 36 & $0.108 \, (0.0914)$ & $0.115 \, (0.0911)$ & $0.112 \, (0.0919)$ & $0.029 \, (0.0280)$ & $-0.004$ \\ 
   & 48 & $0.106 \, (0.0705)$ & $0.071 \, (0.0974)$ & $0.103 \, (0.0757)$ & $0.031 \, (0.0350)$ & $0.000$ \\ 
   \hline
50 & 24 & $0.124 \, (0.0813)$ & $0.109 \, (0.1045)$ & $0.118 \, (0.0879)$ & $0.016 \, (0.0355)$ & $-0.005$ \\ 
   & 36 & $0.126 \, (0.0818)$ & $0.134 \, (0.0878)$ & $0.136 \, (0.0704)$ & $0.027 \, (0.0276)$ & $0.003$ \\ 
   & 48 & $0.101 \, (0.0732)$ & $0.109 \, (0.0767)$ & $0.115 \, (0.0763)$ & $0.020 \, (0.0245)$ & $0.000$ \\ 
   \hline
100 & 24 & $0.117 \, (0.0895)$ & $0.135 \, (0.0973)$ & $0.140 \, (0.0866)$ & $0.019 \, (0.0257)$ & $0.011$ \\ 
   & 36 & $0.113 \, (0.0853)$ & $0.105 \, (0.1033)$ & $0.139 \, (0.0828)$ & $0.021 \, (0.0312)$ & $0.012$ \\ 
   & 48 & $0.111 \, (0.0762)$ & $0.143 \, (0.0853)$ & $0.114 \, (0.0699)$ & $0.031 \, (0.0306)$ & $-0.001$ \\ 
   \hline
\end{tabular}
\caption{Monte carlo value estimates for offline simulations with $\gamma = 0.9$.} 
\label{vl.sim1.1}
\end{table}
The policies estimated using V-learning produce better outcomes 
than the observational policy and the policy estimated using GGQ.  
V-learning produces the best outcomes using Gaussian basis 
functions. 

Next, we simulate cohorts of patients with type 1 diabetes to mimic the 
mHealth study of \cite{maahs2012outpatient}. \cite{maahs2012outpatient} 
followed a small sample of youths with type 1 diabetes and recorded 
data at a fine granularity using mobile devices. Blood glucose levels were 
tracked in real time using continuous glucose monitoring, 
physical activity was measured continuously using accelerometers, 
and insulin injections were logged by an insulin pump. Dietary data 
were recorded by 24-hour recall over phone interviews.

In our simulation study, we divide each day of follow-up into 60 minute intervals. 
Thus, for one day of follow-up, we observe $T = 24$ time points per simulated patient
and a treatment decision is made every hour. 
Our hypothetical mHealth study is designed to estimate an optimal 
dynamic treatment regime for the timing of insulin injections based on 
patient blood glucose, physical activity, and dietary intake with the goal 
of controlling future blood glucose as close as possible to the optimal range. 
To this end, we define the utility at time $t$ as a weighted 
sum of hypo- and hyperglycemic episodes in the 60 minutes preceding and 
following time $t$. 
Weights are $-3$ when $\mathrm{glucose} \le 70$ (hypoglycemic), 
$-2$ when $\mathrm{glucose} > 150$ (hyperglycemic), 
$-1$ when $70 < \mathrm{glucose} \le 80$ or $120 < \mathrm{glucose} \le 150$
(borderline hypo- and hyperglycemic), and $0$ when $80 < \mathrm{glucose} \le 120$ 
(normal glycemia).  
Utility at each time point ranges from $-6$ to 0 with 
larger utilities (closer to 0) being more preferable. 
For example, a patient who presents with an average blood glucose of 155 mg/dL 
over time interval $t - 1$, takes an action to correct their hyperglycemia, 
and presents with an average blood glucose of 145 mg/dL over time interval $t$ would 
receive a utility of $U^t = -3$. 
Weights were chosen to reflect the relative clinical consequences of high and low blood glucose. 
For example, acute hypoglycemia, characterized by blood glucose levels 
below 70 mg/dL, is an emergency that can result in coma or death.  

Simulated data are generated as follows. At each time point, patients 
are randomly chosen to receive an insulin injection with probability 
$0.3$, consume food with probability $0.2$, partake in mild physical 
activity with probability $0.4$, and partake in moderate physical 
activity with probability $0.2$. Grams of food intake and counts of 
physical activity are generated from normal distributions with parameters 
estimated from the data of \cite{maahs2012outpatient}. 
Initial blood glucose level for each patient is drawn from a normal distribution with mean 100 and 
standard deviation 25. Define the covariates for patient $i$ collected at time $t$ by 
$(\mathrm{Gl}_i^t, \mathrm{Di}_i^t, \mathrm{Ex}_i^t)^\T$, 
where $\mathrm{Gl}_i^t$ is average blood glucose level, 
$\mathrm{Di}_i^t$ is total dietary intake, and $\mathrm{Ex}_i^t$ is total counts 
of physical activity as would be measured by an accelerometer. Glucose levels evolve according to 
\begin{equation} \label{t1d.genmod}
\mathrm{Gl}^t = \mu (1 - \alpha_1) + \alpha_1 \mathrm{Gl}^{t-1} + \alpha_2 \mathrm{Di}^{t-1} + \alpha_3 
\mathrm{Di}^{t-2} + \alpha_4 \mathrm{Ex}^{t-1} + \alpha_5 \mathrm{Ex}^{t-2} 
 + \alpha_6 \mathrm{In}^{t-1} + \alpha_7 \mathrm{In}^{t-2} + e,
\end{equation}
where $\mathrm{In}^t$ is an indicator of an insulin injection received at time $t$ and $e \sim N(0, \sigma^2)$. 
We use the parameter vector $\alpha = (\alpha_1, \ldots, \alpha_7)^\T = 
(0.9, 0.1, 0.1, -0.01, -0.01, -2, -4)^\T$, $\mu = 100$, and $\sigma = 5.5$ 
based on a linear model fit to the data of \cite{maahs2012outpatient}. 
The known lag-time in the effect of insulin is reflected by $\alpha_6 = -2$ and $\alpha_7 = -4$.  
Selecting $\alpha_1 < 1$ ensures the existence of a stationary distribution.  
 
We define the state vector for patient $i$ at time $t$ to contain 
average blood glucose, total dietary intake, and total physical activity 
measured over previous time intervals; we include blood glucose and physical activity 
for the previous two time intervals and dietary intake 
for the previous four time intervals. 
Let $n$ denote number of patients and $T$ denote number of time points 
per patient. Our choices for $n$ and $T$ are based on what is feasible for an 
mHealth outpatient study \citep[dietary data was collected on two days by][]{maahs2012outpatient}. 
For each replication, the optimal treatment regime is estimated 
with V-learning using three different types of basis functions and GGQ. 
The generative model for insulin treatment is not assumed to be known 
and we estimate it using logistic regression. 
We record mean outcomes in an independent sample of 100 patients followed for 100 time points 
with treatments generated according to each estimated optimal regime. 
Simulation results (estimated values under each regime and Monte Carlo standard errors along with observed values) 
are found in Table~\ref{vl.sim2.1}. 
\begin{table}[h!]
\centering
\scalebox{0.9}{
\begin{tabular}{cc|ccccc}
   \hline
$n$ & $T$ & Linear VL & Polynomial VL & Gaussian VL & GGQ & Observed \\ 
   \hline
25 & 24 & $-2.716 \, (1.2015)$ & $-2.335 \, (0.9818)$ & $-2.018 \, (1.2011)$ & $-3.870 \, (0.9225)$ & $-2.316$ \\ 
   & 36 & $-2.700 \, (1.2395)$ & $-2.077 \, (1.0481)$ & $-1.760 \, (0.8468)$ & $-3.644 \, (0.8745)$ & $-2.261$ \\ 
   & 48 & $-2.496 \, (1.1986)$ & $-2.236 \, (1.1978)$ & $-1.751 \, (0.9887)$ & $-2.405 \, (1.1025)$ & $-2.365$ \\ 
   \hline
50 & 24 & $-2.545 \, (1.1865)$ & $-2.069 \, (1.0395)$ & $-1.605 \, (0.8064)$ & $-3.368 \, (1.0186)$ & $-2.263$ \\ 
   & 36 & $-2.644 \, (1.1719)$ & $-2.004 \, (0.9074)$ & $-1.778 \, (0.8496)$ & $-3.099 \, (0.9722)$ & $-2.336$ \\ 
   & 48 & $-2.469 \, (1.1635)$ & $-2.073 \, (0.9870)$ & $-2.102 \, (1.2078)$ & $-2.528 \, (0.9571)$ & $-2.308$ \\ 
   \hline
100 & 24 & $-2.350 \, (1.1171)$ & $-2.128 \, (1.0520)$ & $-1.612 \, (0.7203)$ & $-3.272 \, (0.8636)$ & $-2.299$ \\ 
   & 36 & $-2.547 \, (1.1852)$ & $-2.116 \, (0.8518)$ & $-1.672 \, (0.8643)$ & $-3.232 \, (0.7951)$ & $-2.321$ \\ 
   & 48 & $-2.401 \, (1.0643)$ & $-2.204 \, (1.0400)$ & $-1.494 \, (0.5413)$ & $-2.820 \, (0.8442)$ & $-2.351$ \\ 
   \hline
\end{tabular}
}
\caption{Monte carlo value estimates for simulated T1D cohorts with $\gamma = 0.9$.} 
\label{vl.sim2.1}
\end{table}
Again, V-learning with Gaussian basis functions performs the best out 
of all methods, generally producing large values and small standard errors. 
V-learning with the linear model underperforms and GGQ underperforms considerably.  

\subsection{Online simulations} \label{sim.online}

In practice, it may be useful for patients to follow 
a dynamic treatment regime that is updated as new data are collected. 
Here we consider a hypothetical study wherein $n$ patients are followed 
for an initial period of $T^\prime$ time points, an optimal policy is estimated, and patients are 
followed for an additional $T - T^\prime$ time points with the estimated optimal policy 
being continuously updated. At each time point, $t \ge T^\prime$, 
actions are taken according to the most recently estimated policy. 
Recall that V-learning produces a randomized decision rule 
from which to sample actions at each time point. When selecting an action based on a 
GGQ policy, we incorporate an $\epsilon$-greedy strategy by selecting the 
action recommended by the estimated policy with probability $1 - \epsilon$ and otherwise 
randomly selecting one of the other actions. 
At the $t$th estimation, we use $\epsilon = 0.5^t$, 
allowing $\epsilon$ to decrease over time to reflect 
increasing confidence in the estimated policy.  
A burn-in period of 50 time points is discarded to ensure that we are 
sampling from a stationary distribution. 
We estimate the first policy after 12 time points and 
a new policy is estimated every 6 time points thereafter. 
After $T$ time points, we estimate the value as the 
average utility over all patients and all time points after the initial period. 

Table~\ref{vl.sim3.1} contains mean outcomes under policies estimated online 
using data generated according to the simple two covariate generative model 
introduced at the beginning of Section~\ref{sim.offline}.
\begin{table}[h!]
\centering
\begin{tabular}{cc|cccc}
   \hline
$n$ & $T$ & Linear VL & Polynomial VL & Gaussian VL & GGQ \\ 
   \hline
25 & 24 & 0.0053 & 0.0149 & $-0.0100$ & $-0.0081$ \\ 
   & 36 & 0.0525 & 0.0665 & 0.0310 & 0.0160 \\ 
   & 48 & 0.0649 & 0.0722 & 0.0416 & 0.0493 \\ 
   \hline
50 & 24 & 0.0164 & 0.0117 & 0.0037 & 0.0058 \\ 
   & 36 & 0.0926 & 0.0791 & 0.0666 & 0.0227 \\ 
   & 48 & 0.1014 & 0.0894 & 0.0512 & 0.0434 \\ 
   \hline
100 & 24 & 0.0036 & $-0.0157$ & 0.0200 & 0.0239 \\ 
   & 36 & 0.0766 & 0.0626 & 0.0907 & 0.0540 \\ 
   & 48 & 0.0728 & 0.0781 & 0.0608 & 0.0818 \\ 
   \hline
\end{tabular}
\caption{Value estimates for online simulations with $\gamma = 0.9$.} 
\label{vl.sim3.1}
\end{table}
There is some variability across $n$ and $T$ regarding which type of basis 
function is best, but V-learning with a polynomial basis generally 
produces the best outcomes. GGQ performs well in large samples but underperforms 
somewhat otherwise.

Next, we study the performance of online V-learning in simulated mHealth studies 
of type 1 diabetes by following the generative model described in (\ref{t1d.genmod}). 
Mean outcomes are found in Table~\ref{vl.sim4.1}. 
\begin{table}[h!]
\centering
\begin{tabular}{cc|cccc}
   \hline
$n$ & $T$ & Linear VL & Polynomial VL & Gaussian VL & GGQ \\ 
   \hline
25 & 24 & $-2.3887$ & $-1.9713$ & $-1.8860$ & $-3.2027$ \\ 
   & 36 & $-2.3784$ & $-2.1535$ & $-1.7857$ & $-3.5127$ \\ 
   & 48 & $-2.2190$ & $-2.0679$ & $-1.6999$ & $-3.2280$ \\ 
   \hline
50 & 24 & $-2.3405$ & $-2.2313$ & $-1.7761$ & $-2.8976$ \\ 
   & 36 & $-2.2829$ & $-2.0922$ & $-1.6016$ & $-3.1589$ \\ 
   & 48 & $-2.1587$ & $-1.9669$ & $-1.5948$ & $-2.8729$ \\ 
   \hline
100 & 24 & $-2.3229$ & $-2.2295$ & $-1.9138$ & $-3.0865$ \\ 
   & 36 & $-2.2927$ & $-2.1608$ & $-1.9030$ & $-3.3483$ \\ 
   & 48 & $-2.2096$ & $-2.0454$ & $-1.8252$ & $-2.9428$ \\ 
   \hline
\end{tabular}
\caption{Value estimates for online estimation of simulated T1D cohorts with $\gamma = 0.9$} 
\label{vl.sim4.1}
\end{table}
Gaussian V-learning 
performs the best out of all methods, with GGQ consistently underperforming. 
Across all variants of V-learning, outcomes improve with increased 
follow-up time. 

Finally, we consider online simulations using individualized policies 
as outlined at the end of Section~\ref{online}. 
Consider the simple two covariate generative model introduced above but 
let state variables evolve according to 
$S_{i, 1}^t = \mu_i (2 A_i^{t-1} - 1) S_{i, 1}^{t-1} + (1/4) S_{i, 1}^{t-1} S_{i, 2}^{t-1} + \epsilon_1^t$ 
and $S_{i, 2}^t = \mu_i (1 - 2 A_i^{t-1}) S_{i, 2}^{t-1} + (1/4) S_{i, 1}^{t-1} S_{i, 2}^{t-1} + \epsilon_2^t$
where $\mu_i$ is a subject-specific term drawn uniformly between 0.4 and 0.9. 
Including $\mu_i$ ensures that the optimal policy differs across patients. 
Table~\ref{vl.sim5.1} contains mean outcomes for online simulation where 
a universal policy is estimated using data from all patients and where 
individualized policies are estimated using only a single patient's data. 
\begin{table}[h!]
\centering
\begin{tabular}{cc|cc}
   \hline
$n$ & $T$ & Universal policy & Patient-specific policy \\ 
   \hline
25 & 24 & 0.0282 & 0.1813 \\ 
   & 36 & 0.1025 & 0.1700 \\ 
   & 48 & 0.0977 & 0.1944 \\ 
   \hline
50 & 24 & 0.0164 & 0.2771 \\ 
   & 36 & 0.0768 & 0.2617 \\ 
   & 48 & 0.0752 & 0.3038 \\ 
   \hline
100 & 24 & 0.0160 & 0.4230 \\ 
   & 36 & 0.0960 & 0.2970 \\ 
   & 48 & 0.1140 & 0.3197 \\ 
   \hline
\end{tabular}
\caption{Value estimates for online V-learning simulations with univeral and patient-specific policies when $\gamma = 0.9$.} 
\label{vl.sim5.1}
\end{table}
Because data is generated in a such a way that the optimal policy 
varies across patients, individualized policies achieve better outcomes 
than universal policies. 

\section{Case study: Type 1 diabetes} \label{vl.data}

Machine learning is currently under consideration in type 1 diabetes through studies to build 
and test a ``closed loop'' system that joins continuous blood glucose 
monitoring and subcutaneous insulin infusion through an underlying algorithm. 
Known as the artificial pancreas, this technology has been shown to be safe in preliminary 
studies and is making headway from small hospital-based safety studies to 
large-scale outpatient effectiveness studies \citep[][]{ly2014overnight, ly2015day}. 
Despite the success of the artificial pancreas, 
the rate of uptake may be limited and widespread use may not occur for many years \citep[][]{kowalski2015pathway}. 
The proposed method may be useful for implementing mHealth interventions for use alongside 
the artificial pancreas or before it is widely available. 

Studies have shown that data on food intake and physical activity 
to inform optimal decision making can be collected in an inpatient setting 
\citep[see, e.g.,][]{cobry2010timing, wolever2011sugars}. 
However, \cite{maahs2012outpatient} demonstrated that rich data on 
the effect of food intake and physical activity can be collected in an outpatient 
setting using mobile technology. Here, we apply the proposed methodology 
to the observational data collected by \cite{maahs2012outpatient}. 

The full data consist of $N = 31$ patients with type 1 diabetes, aged 12--18. 
Glucose levels were monitored using continuous glucose monitoring and physical activity tracked 
using accelerometers for five days. Dietary data were self-reported by the 
patient in telephone-based interviews for two days. 
Patients were treated using either an insulin pump or multiple daily insulin injections. 
We use data on a subset of $n = 14$ patients treated with an insulin pump 
for whom full follow-up is available on days when dietary information was recorded. 
This represents 28 patient-days of data, with which we use V-learning to estimate an 
optimal treatment policy. 

The setup closely follows the simulation experiments in Section~\ref{sim.offline}. 
Patient state at each time, $t$, is taken to be average glucose level and total 
counts of physical activity over the two previous 60 minute intervals and 
total food intake in grams over the four previous 60 minute intervals. 
The goal is to learn a policy to determine when to 
administer insulin injections based on prior blood glucose, dietary intake, 
and physical activity. The utility at 
time $t$ is a weighted sum of glycemic events over the 60 minutes preceding 
and following time $t$ with weights defined in Section~\ref{sim.offline}. 
A treatment regime with large value will minimize the number 
of hypo- and hyperglycemic episodes weighted to reflect the clinical importance of each. 
We note that because $V(\pi, \bs; \theta^\pi)$ is linear in $\theta^\pi$, we can evaluate 
$\widehat{V}_{n, \widehat{\mathcal{R}}}(\pi)$ with only the mean of $\Phi(\bS)$ under $\mathcal{R}$. 
These were estimated from the data. Because we cannot simulate data following 
a given policy to estimate its value, we report the parametric value estimate 
$\widehat{V}_{n, \widehat{\mathcal{R}}}\left(\widehat{\pi}_n\right)$. 
Interpreting the parametric value estimate is difficult because of the effect 
the discount factor has on estimated value. We cannot compare 
parametric value estimates to mean outcomes observed in the data. Instead,
we use $\mathbb{E}_n \sum_{t \ge 0} \gamma^t U^t$ as
an estimate of value under the observational policy.  

We estimate optimal treatment strategies for two different action spaces. 
In the first, the only decision made at each time is whether or 
not to administer an insulin injection, i.e., the action space contains 
a single binary action. In the second, the action space contains 
all possible combinations of insulin injection, physical activity, and 
food intake. This corresponds to a hypothetical mHealth intervention where 
insulin injections are administered via an insulin pump and suggestions for 
physical activity and food intake are administered via a mobile app. 

Table~\ref{ccat} contains parametric value estimates for policies estimated 
using V-learning for the two action spaces outlined above with different 
basis functions and discount factors. 
\begin{table}[h!]
\centering
\begin{tabular}{l|l|ccc}
   \hline
Action space & Basis & $\gamma = 0.7$ & $\gamma = 0.8$ & $\gamma = 0.9$ \\ 
   \hline
Binary & Linear & $-6.20$ & $-9.35$ & $-15.99$ \\ 
   & Polynomial & $-3.91$ & $-9.03$ & $-17.50$ \\ 
   & Gaussian & $-3.44$ & $-13.09$ & $-25.52$ \\ 
   \hline
Multiple & Linear & $-6.47$ & $-9.92$ & $-0.49$ \\ 
   & Polynomial & $-2.44$ & $-6.80$ & $-14.48$ \\ 
   & Gaussian & $-8.45$ & $-3.58$ & $-21.18$ \\ 
   \hline 
\multicolumn{2}{l|}{Observational policy} & $-6.77$ & $-11.28$ & $-21.79$ \\ \hline
\end{tabular}
\caption{Parametric value estimates for V-learning applied to type 1 diabetes data.} 
\label{ccat}
\end{table}
These results indicate that improvements in glycemic control can come from 
personalized and dynamic treatment strategies that account 
for food intake and physical activity. Improvement results 
from a dynamic insulin regimen (binary action space) and in most cases, further improvement 
results from a comprehensive mHealth intervention including suggestions for 
diet and exercise delivered via mobile app 
in addition to insulin therapy (multiple action space). 
When considering multiple actions, 
the policy estimated using a polynomial basis and $\gamma = 0.7$ 
achieves a 64\% increase in value and the policy estimated using 
a Gaussian basis and $\gamma = 0.8$ achieves a 68\% increase in value 
over the observational policy. 

Finally, we use an example hyperglycemic patient to illustrate how an estimated policy 
would be applied in practice. 
One patient in the data presented at a specific time with an average blood glucose 
of 229 mg/dL over the previous hour and an average blood glucose of 283 mg/dL 
over the hour before that. The policy estimated with $\gamma = 0.7$ and a polynomial 
basis recommends each action according to the probabilities in Table~\ref{ccat.ex}. 
\begin{table}[h!]
\centering
\begin{tabular}{lc}
   \hline
Action & Probability \\ 
   \hline
No action & $< 0.0001$ \\ 
  Physical activity & $< 0.0001$ \\ 
  Food intake & $< 0.0001$ \\ 
  Food and activity & $< 0.0001$ \\ 
  Insulin & 0.7856 \\ 
  Insulin and activity & 0.2143 \\ 
  Insulin and food & 0.0002 \\ 
  Insulin, food, and activity & $< 0.0001$ \\ 
   \hline
\end{tabular}
\caption{Probabilities for each action as recommended by estimated policy for one example patient.} 
\label{ccat.ex}
\end{table}
Because this patient presented with blood glucose levels that are higher than the optimal range, 
the policy recommends actions that would lower the patient's blood glucose levels, 
assigning a probability of 0.79 to insulin and a probability of 0.21 to insulin 
combined with activity.

\section{Conclusion} \label{vl.conc}

The emergence of mHealth has provided great potential for the estimation 
and implementation of dynamic treatment regimes. 
Mobile technology can be used both in the collection of 
rich longitudinal data to inform decision making and the delivery of 
deeply tailored interventions. The proposed method, V-learning, addresses 
a number of challenges associated with estimating dynamic treatment regimes in 
mHealth applications. V-learning directly estimates a policy which maximizes 
the value over a class of policies and requires minimal assumptions on the 
data-generating process. Furthermore, V-learning permits 
estimation of a randomized decision rule which can be used in place 
of existing strategies (e.g., $\epsilon$-greedy) to encourage exploration in online 
estimation. A randomized decision rule can also provide patients 
with multiple treatment options. 
Estimation of an optimal policy for different populations 
can be handled through the use of different reference distributions. 

V-learning and mobile technology have the potential to improve patient 
outcomes in a variety of clinical areas. 
We have demonstrated, for example, that the proposed method can be 
used to estimate treatment regimes to reduce 
the number of hypo- and hyperglycemic episodes in patients 
with type 1 diabetes. The proposed method could also be useful 
for other mHealth applications as well as applications outside of 
mHealth. For example, V-learning could be used to estimate 
dynamic treatment regimes for chronic illnesses using electronic health records data. 
Future research in this area may include increasing flexibility through 
use of a semiparametric model for the state-value function. Alternatively, 
nonlinear models for the state-value function may be informed by 
underlying theory or mathematical models of the system of interest. 
Data-driven selection of tuning parameters for the proposed method may 
help to improve performance. 
Finally, accounting for patient availability and feasibility of 
a sequence of treatments can be done by setting constraints on the 
class of policies. This will ensure that the resulting mHealth intervention 
is able to be implemented and that the recommended 
decisions are consistent with domain knowledge. 

\section*{Appendix}

\begin{proof}[Proof of Lemma~\ref{lemma}]
Let $\pi$ be an arbitrary policy and $\gamma \in (0, 1)$ a fixed constant. 
Suppose we observe a state $\bS^t = \bs^t$ at time $t$ and let $\overline{a}^{t-1} = (a^1, \ldots, a^{t-1})$ 
be the sequence of actions resulting in $\bS^t = \bs^t$, i.e., $\bS^{* t}(\overline{a}^{t-1}) = \bs^t$. 
Let $\overline{a}^{k+1} = (a^t, \ldots, a^{t+k}) \in \mathcal{A}^{k+1}$ be a potential sequence of actions 
taken from time $t$ to time $t+k$. We have that 
\begin{eqnarray*}
V(\pi, \bs^t) & = & \sum_{k \ge 0} \gamma^k \mathbb{E} \left\{ U^{*(t + k)} (\pi) \bigg| \bS^t = \bs^t \right\} \\
 & = & \sum_{k \ge 0} \gamma^k \mathbb{E} \left(  
 \sum_{\overline{a}^{t+k} \in \mathcal{A}^{t+k}} U^{*(t + k)} (\overline{a}^{t + k}) 
 \prod_{v = t}^{t + k} 1\left[ \xi_\pi^v \left\{ \bS^{*v}(\overline{a}^{v-1})\right\} = a^v \right] \bigg| \bS^t = \bs^t \right) \\
 & = & \sum_{k \ge 0} \gamma^k  
 \sum_{\overline{a}^{k+1} \in \mathcal{A}^{k+1}} U^{*(t + k)} (\overline{a}^{t-1}, \overline{a}^{k+1}) 
 \Bigg\{ \prod_{v = t}^{t + k} \mathbb{E} \left( 1\left[ \xi_\pi^v \left\{ \bS^{*v}(\overline{a}^{v-1})\right\} = a^v\right] 
  \bigg| \bS^t = \bs^t \right) \Bigg\} \\
& = & \sum_{k \ge 0} \gamma^k 
 \sum_{\overline{a}^{k+1} \in \mathcal{A}^{k+1}} U^{*(t + k)} (\overline{a}^{t-1}, \overline{a}^{k+1}) 
 \prod_{v = t}^{t + k} \pi\left\{a^v; \bS^{*v}(\overline{a}^{v-1})\right\} 
 \prod_{v = t}^{t + k} \frac{\mu^v\left\{a^v; \bS^{*v}(\overline{a}^{v-1})\right\}}
{\mu^v\left\{a^v; \bS^{*v}(\overline{a}^{v-1})\right\}}\\ 
& = & \sum_{k \ge 0} \gamma^k \mathbb{E} \left[ U^{t + k} \left\{\prod_{v = 0}^k \frac{\pi(a^{t + v}; \bs^{t + v})}
 {\mu^{t + v}(a^{t + v}; \bs^{t + v})}\right\} \bigg| \bS^t = \bs^t \right], 
\end{eqnarray*} 
where we let $\pi(a^t; \bs^t) = 0$ for all $a^t$ and $\bs^t$ whenever $t > T^*(\pi)$. 
The last equality uses the consistency and strong ignorability assumptions. 
\end{proof}

\begin{proof}[Proof of Theorem~\ref{main.theta}]
Proof of part~\ref{main1}: We first note that $\theta_0^\pi$ must solve 
$$
0 = \mathbb{E} \left( \frac{\pi(A^t; \bS^t)}{\mu^t(A^t; \bS^t)} 
\left[ U^t + \left\{ \gamma \Phi(\bS^{t + 1}) - \Phi(\bS^t) \right\}^\T \theta^\pi \right] 
\Phi(\bS^t) \right), 
$$
or 
$$
\mathbb{E} \left[ \frac{\pi(A^t; \bS^t)}{\mu^t(A^t; \bS^t)} \Phi(\bS^t) 
\left\{ \Phi(\bS^t) - \gamma \Phi(\bS^{t + 1})\right\}^\T \right] \theta^\pi 
 = \mathbb{E}\left\{ \frac{\pi(A^t; \bS^t)}{\mu^t(A^t; \bS^t)} U^t \Phi(\bS^t) \right\},
$$
which is equivalent to $w_1(\pi)\theta^\pi = w_2(\pi)$ where 
$w_2(\pi) = \mathbb{E}\left\{ \pi(A^t; \bS^t) \mu^t(A^t; \bS^t)^{-1} U^t \Phi(\bS^t) \right\}$. 
We have that 
$$
\left\| \mathbb{E} \left\{ \frac{\pi(A^t; \bS^t)}{\mu^t(A^t; \bS^t)} U^t \Phi(\bS^t) \right\} \right\|
 \le c_0^{-1} \left( \mathbb{E} |U^t|^2 \right)^{1/2} \left( \mathbb{E} \| \Phi(\bS^t)\|^2 \right)^{1/2} < \infty, 
$$
by assumption~\ref{positivity.assume}, part~\ref{moment.assume} of assumption~\ref{rho.assume} and the Cauchy--Schwarz inequality. 
Let $c \in \mathbb{R}^q$ be arbitrary and note that 
\begin{multline*}
\mathbb{E} \left\{ \frac{\pi(A^t; \bS^t)}{\mu^t(A^t; \bS^t)} c^\T \Phi(\bS^t) \Phi(\bS^{t + 1})^\T c \right\} \\
 \le \left[ \mathbb{E} \left\{ \frac{\pi(A^t; \bS^t)}{\mu^t(A^t; \bS^t)} c^\T \Phi(\bS^t)^{\otimes 2} c \right\} \cdot 
 \mathbb{E} \left\{ \frac{\pi(A^t; \bS^t)}{\mu^t(A^t; \bS^t)} c^\T \Phi(\bS^{t + 1})^{\otimes 2} c \right\} \right]^{1/2},  
\end{multline*}
by the Cauchy--Schwarz inequality, 
where $u^{\otimes 2} = u u^\T$. This implies that 
\begin{eqnarray*}
c^\T w_1(\pi) c & \ge & \mathbb{E} \left\{ \frac{\pi(A^t; \bS^t)}{\mu^t(A^t; \bS^t)} c^\T \Phi(\bS^t)^{\otimes 2} c \right\} \\
 & & - \mathbb{E} \left\{ \frac{\pi(A^t; \bS^t)}{\mu^t(A^t; \bS^t)} c^\T \Phi(\bS^t)^{\otimes 2} c \right\}^{1/2} 
 \mathbb{E} \left\{ \gamma^2 \frac{\pi(A^t; \bS^t)}{\mu^t(A^t; \bS^t)} c^\T \Phi(\bS^{t + 1})^{\otimes 2} c \right\}^{1/2} \\
 & = & A - A^{1/2} B^{1/2} \\ 
 & = & A^{1/2} (A^{1/2} - B^{1/2}) \\
 & = & \frac{A^{1/2}(A - B)}{A^{1/2} + B^{1/2}}, 
\end{eqnarray*}
where we simplify notation by defining 
$A = \mathbb{E} \left\{ \pi(A^t; \bS^t) \mu^t(A^t; \bS^t)^{-1} c^\T \Phi(\bS^t)^{\otimes 2} c \right\}$ 
and $B = \mathbb{E} \left\{ \gamma^2 \pi(A^t; \bS^t) \mu^t(A^t; \bS^t)^{-1} c^\T \Phi(\bS^{t + 1})^{\otimes 2} c \right\}$. 
We have that 
\begin{eqnarray*}
A^{1/2} + B^{1/2} & \le & c_0^{-1/2} \|c\| \left\{ \mathbb{E} \| \Phi(\bS^t) \|^2 \right\}^{1/2} 
 + c_0^{-1/2} \|c\| \left\{ \mathbb{E} \| \Phi(\bS^{t + 1}) \|^2 \right\}^{1/2} \\
 & = & 2 c_0^{-1/2} \|c\| \left\{ \mathbb{E} \| \Phi(\bS^t) \|^2 \right\}^{1/2} \\
 & < & \infty, 
\end{eqnarray*}
by the Cauchy--Schwarz inequality, the fact that 
$\mathbb{E}\|\Phi(\bS^t)\|^2 = \mathbb{E}\|\Phi(\bS^{t + 1})\|^2$ 
by time-homogeneity, and part~\ref{moment.assume} of 
assumption~\ref{rho.assume}. 
Also, $A \ge A - B$ and $A - B \ge c_1 \|c\|^2$ by assumption~\ref{invertibility.assume}. 
Thus, 
\begin{eqnarray*}
A - A^{1/2}B^{1/2} & \ge & \frac{c_1^{3/2} \|c\|^3}{2 c_0^{-1/2} \|c\| \left\{ \mathbb{E}\| \Phi(\bS^t)\|^2\right\}^{1/2}} \\
 & = & \frac{c_0^{1/2} c_1^{3/2} \|c\|^2}{2\left\{ \mathbb{E} \| \Phi(\bS^t)\|^2 \right\}^{1/2}}, 
\end{eqnarray*}
which finally implies that $w_1(\pi)$ is invertible and thus 
$\theta_0^\pi = w_1(\pi)^{-1} w_2(\pi)$ is well-defined uniformly over $\pi \in \Pi$. 
Using the fact that $c^\T w_1(\pi) c \ge k_0 \|c\|^2$ for a constant $k_0 > 0$, 
we can show that $\|w_1(\pi)^{-1}\| \le k_1^{-1}$ for some constant $k_1 > 0$, where $\| \cdot \|$ is the usual matrix norm when applied 
to a matrix. Therefore, $\|\theta_0^\pi\| \le k_1^{-1} \|w_2(\pi)\| \le c_0^{-1} k_1^{-1} 
\left\{ \mathbb{E}(U^t)^2\right\}^{1/2} \left\{ \mathbb{E}\|\Phi(\bS^t)\|^2 \right\}^{1/2} < \infty$. 
Finally, it follows from assumptions~\ref{invertibility.assume} and \ref{continuity.assume} that 
$\sup_{\| \beta_1 - \beta_2 \| \le \delta} \| \theta^{\pi_{\beta_1}} - \theta^{\pi_{\beta_2}} \| \rightarrow 0$ 
as $\delta \downarrow 0$. 

Proof of part~\ref{main2}: Define 
$$
\mathcal{G} = \left\{ \Phi(\bs^t) \Phi(\bs^t)^\T / \mu^t(a^t; \bs^t), \gamma \Phi(\bs^t) \Phi(\bs^{t + 1})^\T / \mu^t(a^t; \bs^t), 
u^t \Phi(\bs^t) / \mu^t(a^t; \bs^t)\right\}.
$$ 
Let $G$ be an envelope for $\mathcal{G}$, for example 
$G(\bs^{t + 1}, a^t, \bs^t) = \max_{g \in \mathcal{G}} g(\bs^{t + 1}, a^t, \bs^t)$. 
By part~\ref{moment.assume} of assumption~\ref{rho.assume}, 
$\mathbb{E} G^{3\rho} < \infty$. 
Part~\ref{fg4} of Lemma~\ref{fg.lem} below gives us that $\mathcal{G}$ is Donsker. 
Since $\Pi$ satisfies $J_{[]}\left\{\infty, \Pi, L_{3\rho}(P)\right\} < \infty$, 
we have that 
$$
\mathcal{F}_1 = \left\{ \frac{\pi(a^t; \bs^t)}{\mu^t(a^t; \bs^t)} \Phi(\bs^t) 
\left\{ \Phi(\bs^t) - \gamma \Phi(\bs^{t + 1}) \right\}^\T : \pi \in \Pi \right\}
$$
satisfies $J_{[]}\left\{\infty, \mathcal{F}_1, L_{3\rho}(P)\right\} < \infty$ by 
parts~\ref{fg1} and \ref{fg2} of Lemma~\ref{fg.lem} below. Moreover, 
$F(a^t, \bs^t) = \|\Phi(\bs^t)\| \cdot \|\Phi(\bs^t) - \gamma \Phi(\bs^{t + 1})\| / \mu^t(a^t; \bs^\T)$
is an envelope for $\mathcal{F}_1$ with $\mathbb{E} F^{3 \rho} < \infty$ 
by assumption~\ref{positivity.assume} and part~\ref{moment.assume} of assumption~\ref{rho.assume}. 
Thus, $\mathcal{F}_1$ is Donsker. 
Let 
$$
\mathcal{F}_2 = \left\{ \frac{\pi(a^t; \bs^t)}{\mu^t(a^t; \bs^t)} u^t \Phi(\bs^t) : \pi \in \Pi \right\}.
$$ 
Similar arguments yield that $\mathcal{F}_2$ is Donsker. 

Now, let $\widehat{A}(\pi) = \{ \mathbb{E}_n f_{1\pi} : f_{1\pi} \in \mathcal{F}_1\}$ 
and $\widehat{B}(\pi) = \{ \mathbb{E}_n f_{2\pi} : f_{2\pi} \in \mathcal{F}_2\}$. 
Let $\widehat{A}^\prime(\pi) = \widehat{A}(\pi) + \lambda_n \widehat{A}(\pi)^{-1}$. 
We have that $\widehat{\theta}^\pi_n = \widehat{A}^\prime(\pi)^{-1} \widehat{B}(\pi)$. 
Thus, 
\begin{eqnarray*}
\sqrt{n} \left( \widehat{\theta}_n^\pi - \theta_0^\pi \right) & = & \sqrt{n} 
\left\{ \widehat{A}^\prime(\pi)^{-1} \widehat{B}(\pi) 
 - \widehat{A}^\prime(\pi)^{-1} \widehat{A}^\prime(\pi) \theta_0^\pi \right\} + o_P(1) \\
 & = & \widehat{A}^\prime(\pi)^{-1} \sqrt{n} \left\{ \widehat{B}(\pi) 
 - \widehat{A}^\prime(\pi) \theta_0^\pi \right\} + o_P(1) \\
 & = & \widehat{A}^\prime(\pi)^{-1} \sqrt{n} \left\{ \widehat{B}(\pi) 
 - \widehat{A}(\pi) \theta_0^\pi \right\} + \widehat{A}^\prime(\pi)^{-1} \sqrt{n} 
 \left\{ \widehat{A}(\pi) - \widehat{A}^\prime(\pi) \right\} \theta_0^\pi + o_P(1) \\
 & = & \widehat{A}^\prime(\pi)^{-1} \sqrt{n} \left\{ \widehat{B}(\pi) 
 - \widehat{A}(\pi) \theta_0^\pi \right\} + o_P(1)
\end{eqnarray*}
where $o_P(1)$ doesn't depend on $\pi$, 
because $\widehat{A}^\prime(\pi)^{-1} \xrightarrow[]{P} w_1(\pi)^{-1} < \infty$ uniformly 
over $\pi \in \Pi$ by assumption~\ref{positivity.assume} and part~\ref{moment.assume} 
of assumption~\ref{rho.assume}, $\sup_{\pi \in \Pi} \| \theta_0^\pi \| < \infty$ 
by part \ref{main1} of this theorem, and $\sqrt{n} \left\{\widehat{A}(\pi) - \widehat{A}^\prime(\pi)\right\}
 = \sqrt{n} \lambda_n \widehat{A}(\pi)^{-1} = o_P(1)$ because $\lambda_n = o_P(n^{-1/2})$. 
Using arguments similar to those in the previous paragraph, one can show that 
$\mathcal{F}_3 = \left\{ f_{2\pi} - f_{1\pi} \theta : 
f_{1\pi} \in \mathcal{F}_1, f_{2\pi} \in \mathcal{F}_2, \pi \in \Pi, \theta \in B_* \right\}$ 
is Donsker, where $B_*$ is any finite collection of elements of $\mathbb{R}^q$. 
By part~\ref{main1} of this theorem, there exists a bounded, closed set $B_0$ such that 
$\theta_0^\pi \in B_0$ for all $\pi \in \Pi$. 
Let $\mathbb{G}_n(\pi, \theta) = \sqrt{n} (\mathbb{E}_n - \mathbb{E})(f_{2\pi} - f_{1\pi} \theta)$. 
Note that 
\begin{eqnarray*}
\sup_{\pi \in \Pi} \| \mathbb{G}_n(\pi, \theta_1) - \mathbb{G}_n(\pi, \theta_2) \| 
 & \le & \sup_{\pi \in \Pi} \| \sqrt{n} (\mathbb{E}_n - \mathbb{E}) f_{1\pi} \| \cdot \| \theta_1 - \theta_2 \| \\
 & \le & R^* \| \theta_1 - \theta_2 \|, 
\end{eqnarray*}
where $R^* = O_P(1)$ by the Donsker property of $\mathcal{F}_1$ and $R^*$ doesn't depend on $\pi$. 
Thus, $\mathbb{G}_n(\pi, \theta)$ is stochastically equicontinuous 
on $B_0$. 
Combined with the Donsker property of $\mathcal{F}_3$ for arbitrary $B_*$, we have that the class 
$\mathcal{F}_4 = \{ f_{2\pi} - f_{1\pi} \theta : f_{1\pi} \in \mathcal{F}_1, f_{2\pi} \in \mathcal{F}_2, \pi \in \Pi, \theta \in B_0\}$ 
is Donsker. Using Slutsky's Theorem, Theorem 11.24 of \cite{kosorok2008introduction}, 
the fact that $\mathcal{F}_1$ is Glivenko--Cantelli, 
and the fact that $\theta_0^\pi = (\mathbb{E} f_{1\pi})^{-1} \mathbb{E}f_{2\pi}$, we have that 
$\sqrt{n}\left( \widehat{\theta}_n^\pi - \theta_0^\pi \right) = \widehat{A}^\prime(\pi)^{-1} \mathbb{G}_n(\pi, \theta_0^\pi) 
\rightsquigarrow w_1(\pi)^{-1} \mathbb{G}_0(\pi)$ in $\ell^\infty(\Pi)$, where 
$\mathbb{G}_0(\pi)$ is a mean zero Gaussian process indexed by $\Pi$ with covariance 
$\mathbb{E} \left\{ \mathbb{G}_0(\pi_1) \mathbb{G}_0(\pi_2) \right\} = w_0(\pi_1, \pi_2)$. 

Proof of part~\ref{main3}: We have that 
\begin{eqnarray*} \sqrt{n} \left\{ \widehat{V}_{n, \widehat{\mathcal{R}}}(\pi) - V_\mathcal{R}(\pi) \right\} 
 & = & \sqrt{n} \mathbb{E}_n \Phi(\bS^t) \left( \widehat{\theta}_n^\pi - \theta_0^\pi \right) \\
 & \rightsquigarrow & \nu^\T w_1(\pi)^{-1} \mathbb{G}_0(\pi) 
\end{eqnarray*}
in $\ell^\infty(\Pi)$ by Slutsky's Theorem. 
\end{proof}

\begin{proof}[Proof of Theorem~\ref{main.beta}]
Proof of part~\ref{main4}: Following part~\ref{main3} of Theorem~\ref{main.theta}, we have that 
$\sup_{\beta \in \mathcal{B}} \left| \widehat{V}_{n, \widehat{\mathcal{R}}}(\pi_\beta) 
 - V_\mathcal{R}(\pi_\beta) \right| \xrightarrow[]{P} 0$. 
Combining this with the unique and well separated maximum condition (assumption \ref{unique.assume}), 
continuity of $V_\mathcal{R}(\pi_\beta)$ in $\beta$, and Theorem~2.12 of \cite{kosorok2008introduction} 
yields the result in part~\ref{main4}. 
Part~\ref{main5} follows from parts~\ref{main2} and \ref{main3} of Theorem~\ref{main.theta}. 
The proof of part~\ref{main6} follows standard arguments.
\end{proof}

\begin{lem} \label{fg.lem}
Let $\mathcal{F}$ and $\mathcal{G}$ be function classes with respective envelopes $F$ and $G$. 
Let $\|F\|_u = \left(\mathbb{E}|F|^u\right)^{1/u}$. 
For any $1 \le r, s_1, s_2 \le \infty$ with $s_1^{-1} + s_2^{-1} = 1$, 
\begin{enumerate}
\item $J_{[]}\{\infty, \mathcal{F} \cdot \mathcal{G}, L_r(P)\} \le 
2 \left( \|F\|_{rs_1} + \|G\|_{rs_2} \right) \left[ J_{[]}\{\infty, \mathcal{F}, L_{rs_1}(P)\} 
 + J_{[]}\{\infty, \mathcal{G}, L_{rs_2}(P)\} \right]$. 
\label{fg1}
\item $J_{[]}\{\infty, \mathcal{F} + \mathcal{G}, L_r(P)\} \le 2 \left[ J_{[]}\{\infty, \mathcal{F}, L_r(P)\} 
 + J_{[]}\{\infty, \mathcal{G}, L_r(P)\} \right]$. 
\label{fg2}
\item For any $0 \le r \le \infty$, $J_{[]}\{\infty, \mathcal{F} \cup \mathcal{G}, L_r(P)\} 
\le \sqrt{\log 2} (\|F\|_r + \|G\|_r) + J_{[]}\{\infty, \mathcal{F}, L_r(P)\} + J_{[]}\{\infty, \mathcal{G}, L_r(P)\}$.
\label{fg3}
\item If $\mathcal{G}$ is a finite class, $J_{[]}\{\infty, \mathcal{G}, L_r(P)\} \le 2 \|G\|_r \sqrt{\log |\mathcal{G}|}$, 
where $|\mathcal{G}|$ denotes the cardinality of $\mathcal{G}$. 
\label{fg4}
\end{enumerate} 
\end{lem}

\begin{proof}[Proof of Lemma~\ref{fg.lem}]
Proof of part~\ref{fg1}: Let $1 \le r, s_1, s_2 \le \infty$ with $s_1^{-1} + s_2^{-1} = 1$ and 
let $(\ell_F, u_F)$ and $(\ell_G, u_G)$ be $L_{rs_1}(P)$ and $L_{rs_2}(P)$ 
$\epsilon$-brackets, respectively. Choose $\ell_F \le f_1, f_2 \le u_F$ and $\ell_G \le g_1, g_2 \le u_G$ 
and consider the bracket for any $f_2 g_2$ defined by $f_1 g_1 \pm \left( F|u_G - \ell_G| + G|u_F - \ell_F|\right)$. 
Note that 
$$
f_1 g_1 + F|u_G - \ell_G| + G|u_F - \ell_F| - f_2 g_2 \ge F|u_G - \ell_G| + G|u_F - \ell_F| - F|g_1 - g_2| - G|f_1 - f_2| \ge 0, 
$$
because $f_2 g_2 - f_1 g_1 = f_2 g_2 + f_2 g_1 - f_2 g_1 - f_1 g_1 \le F|g_1 - g_2| + G|f_1 - f_2|$.
Similarly, $f_2 g_2 + F|u_G - \ell_G| + G|u_F - \ell_F| - f_1 g_1 \ge 0$. 
Thus, these brackets hold all $f_2 g_2$ for $f_2 \in (\ell_F, u_F)$ and $g_2 \in (\ell_G, u_G)$. 
Now, $\| F|u_G - \ell_G| + G|u_F - \ell_F| \|_r \le \|F\|_{rs_1} \epsilon + \|G\|_{rs_2} \epsilon$
by Minkowski's inequality and H\"{o}lder's inequality and it follows that 
$$
N_{[]}\left\{2\epsilon(\|F\|_{rs_1} + \|G\|_{rs_2}), \mathcal{F} \cdot \mathcal{G}, L_r(P)\right\} 
\le N_{[]}\left\{\epsilon, \mathcal{F}, L_{rs_1}(P)\right\} N_{[]}\left\{\epsilon, \mathcal{G}, L_{rs_2}(P)\right\}. 
$$ 
Next we note that 
$$
N_{[]} \left\{ \epsilon, \mathcal{F} \cdot \mathcal{G}, L_r(P)\right\} 
 \le N_{[]}\left\{ \frac{\epsilon}{2(\|F\|_{rs_1} + \|G\|_{rs_2})}, \mathcal{F}, L_{rs_1}(P) \right\} 
 N_{[]}\left\{ \frac{\epsilon}{2(\|F\|_{rs_1} + \|G\|_{rs_2})}, \mathcal{G}, L_{rs_2}(P) \right\}
$$
and thus 
\begin{eqnarray*}
J_{[]}\left\{ \infty, \mathcal{F} \cdot \mathcal{G}, L_r(P) \right\} 
& \le & \int_0^{2 \|F\|_{rs_1} \|G\|_{rs_2}} \sqrt{\log N_{[]}\left\{ \frac{\epsilon}{2(\|F\|_{rs_1} + \|G\|_{rs_2})}, \mathcal{F}, L_{rs_1}(P) \right\}} \mathrm{d} \epsilon \\
 & & + \int_0^{2 \|F\|_{rs_1} \|G\|_{rs_2}} \sqrt{\log N_{[]}\left\{ \frac{\epsilon}{2(\|F\|_{rs_1} + \|G\|_{rs_2})}, \mathcal{G}, L_{rs_2}(P) \right\}} \mathrm{d} \epsilon \\
& \le & 2 \left( \|F\|_{rs_1} + \|G\|_{rs_2} \right) \left[ J_{[]}\{\infty, \mathcal{F}, L_{rs_1}(P)\} 
 + J_{[]}\{\infty, \mathcal{G}, L_{rs_2}(P)\} \right]. 
\end{eqnarray*}

The proof of part~\ref{fg2} follows from Lemma 9.25 part (i) of \cite{kosorok2008introduction} after a change of variables. 
Proof of part~\ref{fg3}: First note that 
$$
N_{[]}\left\{\epsilon, \mathcal{F} \cup \mathcal{G}, L_r(P)\right\} \le N_{[]}\left\{\epsilon, \mathcal{F}, L_r(P)\right\} 
 + N_{[]}\left\{\epsilon, \mathcal{G}, L_r(P)\right\}, 
$$ 
whence it follows that 
\begin{eqnarray*}
J_{[]}\left\{\infty, \mathcal{F} \cup \mathcal{G}, L_r(P)\right\} & = & \int_0^{2(\|F\|_r + \|G\|_r)} 
 \sqrt{\log N_{[]}\left\{ \epsilon, \mathcal{F} \cup \mathcal{G}, L_r(P)\right\}} \mathrm{d} \epsilon \\
 & \le & \int_0^{2(\|F\|_r + \|G\|_r)} \sqrt{\log \left[ N_{[]}\left\{\epsilon, \mathcal{F}, L_r(P)\right\} 
 + N_{[]}\left\{\epsilon, \mathcal{G}, L_r(P)\right\}\right]} \mathrm{d} \epsilon \\
 & \le & \int_0^{2(\|F\|_r + \|G\|_r)} \sqrt{\log 2 + \log N_{[]}\left\{\epsilon, \mathcal{F}, L_r(P)\right\} 
 + \log N_{[]}\left\{\epsilon, \mathcal{G}, L_r(P)\right\}} \mathrm{d} \epsilon \\
 & \le & \int_0^{2(\|F\|_r + \|G\|_r)} \sqrt{\log 2} \mathrm{d} \epsilon + J_{[]}\left\{ \infty, \mathcal{F}, L_r(P)\right\} 
 + J_{[]}\left\{\infty, \mathcal{G}, L_r(P)\right\},
\end{eqnarray*}
where the second inequality uses the fact that $a + b \le 2ab$ for all $a,b \ge 1$.

Proof of part~\ref{fg4}: If $\mathcal{G}$ is finite, then 
$N_{[]}\left\{ \epsilon, \mathcal{G}, L_r(P)\right\} \le |\mathcal{G}|$. 
Thus, 
\begin{eqnarray*}
J_{[]}\left\{\infty, \mathcal{G}, L_r(P)\right\} & = & \int_0^{2\|G\|_r} 
\sqrt{\log N_{[]}\left\{\epsilon, \mathcal{G}, L_r(P)\right\}} \mathrm{d}\epsilon \\
 & \le & \int_0^{2\|G\|_r} \sqrt{\log |G|} \mathrm{d}\epsilon,
\end{eqnarray*}
which completes the proof. 
\end{proof}

\begin{lem} \label{policy.class}
Define the class of functions 
$$
\Pi = \left\{ \pi_{\tilde{\beta}}(a; \bs) = \frac{a_J + \sum_{j = 1}^{J - 1} a_j \mathrm{exp}(\bs_j^\T \beta_j)} 
 { 1 + \sum_{j = 1}^{J - 1} \mathrm{exp}(\bs_j^\T \beta_j)} : \tilde{\beta} = (\beta_1^\T, \ldots, \beta_{J - 1}^\T)^\T, 
\tilde{\beta} \in \mathcal{B} \subset \mathbb{R}^{p(J - 1)} \right\}
$$ 
for a compact set $\mathcal{B}$ and $2 \le J < \infty$ where $a = (a_1, \ldots, a_J)^\T$.  
Then, there exists a $b_0 < \infty$ such that for any $1 \le r \le \infty$, $J_{[]}\{\infty, \Pi, L_r(P)\} 
\le b_0 \|\bS\|_r \sqrt{p(J - 1) \pi}$, which is finite whenever $\|\bS\|_r < \infty$. 
Furthermore, $\sup_{\|\tilde{\beta}_1 - \tilde{\beta}_2\| \le \delta} \mathbb{E} \| \pi_{\tilde{\beta}_1}(A; \bS) - \pi_{\tilde{\beta}_2}(A; \bS) \| 
\rightarrow 0$ as $\delta \downarrow 0$. 
\end{lem}

\begin{proof}[Proof of Lemma~\ref{policy.class}]
For $\tilde{\beta}_1, \tilde{\beta}_2 \in \mathcal{B}$, define $d(\tilde{\beta}_1, \tilde{\beta}_2) = 
\mathrm{max}_{1 \le j \le J - 1} \| \tilde{\beta}_{1j} - \tilde{\beta}_{2j} \|$ and 
$b_0 = \sup_{\tilde{\beta}_1, \tilde{\beta}_2 \in \mathcal{B}} \|\tilde{\beta}_1 - \tilde{\beta}_2\| < 0$ because $\mathcal{B}$ is compact. 
By the mean value theorem, for any $\tilde{\beta_1}, \tilde{\beta}_2 \in \mathcal{B}$, there exists a point $\tilde{\beta}_*$ 
on the line segment between $\tilde{\beta}_1$ and $\tilde{\beta}_2$ such that 
$$
\pi_{\tilde{\beta}_1}(a; \bs) - \pi_{\tilde{\beta}_2}(a; \bs) = \frac{1}{1 + \sum_{j = 1}^{J - 1}\mathrm{exp}(\bs^\T \tilde{\beta}_{*j})} 
\left[ \sum_{j = 1}^{J - 1} \left\{ a_j - \pi_{\tilde{\beta}_*}(a; \bs) \right\} 
\mathrm{exp}(\bs^\T \tilde{\beta}_{*j}) \bs^\T (\tilde{\beta}_{1j} - \tilde{\beta}_{2j}) \right],  
$$
which implies that 
\begin{equation} \label{policy.proof.eq}
|\pi_{\tilde{\beta}_1}(a; \bs) - \pi_{\tilde{\beta}_2}(a; \bs)| \le \|\bs\| d(\tilde{\beta}_1, \tilde{\beta}_2).
\end{equation}
It follows from equation~(\ref{policy.proof.eq}) that assumption~\ref{continuity.assume} 
holds for this particular class of policies. 
Now, $N_{[]}\left\{2\epsilon \|\bS\|_r, \Pi, L_r(P)\right\} \le N(\epsilon, \mathcal{B}, d)$ 
by Theorem~9.23 of \cite{kosorok2008introduction}. Furthermore, $N(\epsilon, \mathcal{B}, d) \le \max\left\{(b_0 / \epsilon)^{p(J - 1)}, 1\right\}$, 
and thus 
\begin{eqnarray*}
J_{[]}\left\{\epsilon, \Pi, L_r(P)\right\} & \le & 2\|\bS\|_r \int_0^{b_0} \sqrt{p(J - 1) \left\{ \log b_0 + \log(1 / \epsilon)\right\} } \mathrm{d}\epsilon \\
 & \le & 2\|\bS\|_r b_0 \sqrt{p(J - 1)} \int_0^1 \sqrt{\log (1 / \epsilon)} \mathrm{d} \epsilon \\
 & = & 2\|\bS\|_r b_0 \sqrt{p(J - 1)} \int_0^\infty u^{1/2} \mathrm{exp}(-u) \mathrm{d} u \\
 & = & \|\bS\|_r b_0 \sqrt{p(J - 1) \pi},  
\end{eqnarray*} 
which proves the result. 
\end{proof}

\bibliographystyle{Chicago}

\bibliography{vlearning}
\end{document}